\documentclass[journal]{IEEEtran}
\usepackage{cite}
\usepackage{url}
\usepackage{float}
\usepackage{multirow}
\usepackage{booktabs}
\usepackage{caption}
\usepackage{amsmath}
\usepackage{amsthm}
\usepackage{amssymb, amsfonts}
\usepackage{algorithm}
\usepackage{algorithmic}
\allowdisplaybreaks[4]
\usepackage{color,graphicx}
\graphicspath{{figure/}}
\usepackage{lettrine}
\usepackage{ulem}

\newtheorem{thm}{Theorem}
\newtheorem{proposition}{Proposition}

\newenvironment{sequation}{\begin{equation}\small}{\end{equation}}

\begin{document}
\title{Bayesian Cycle-Consistent Generative Adversarial Networks via Marginalizing Latent Sampling}
\author{
Haoran~You, 
Yu~Cheng,
Tianheng~Cheng,
Chunliang~Li,
and
Pan~Zhou 
\thanks{
Haoran You, Pan Zhou and Tianheng Cheng are with School of Electrical Information and Communication Engineering, Huazhong University of Science \& Technology, Wuhan, 430074, China. Email: \{ranery,vic,panzhou,@alumni.hust.edu.cn\}

Yu Cheng is with Microsoft AI \& Research, Redmond, Washington, 98052. Email: yu.cheng@microsoft.com

Chunliang Li is with Machine Learning Department, Carnegie Mellon University, Pittsburgh, 15213. Email: chunlial@cs.cmu.edu
}
\vspace{-.4cm}
}

\IEEEtitleabstractindextext{
\begin{abstract}
Recent techniques built on Generative Adversarial Networks (GANs), such as Cycle-Consistent GANs, are able to learn mappings among different domains built from unpaired datasets, through min-max optimization games between generators and discriminators. However, it remains challenging to stabilize the training process and thus cyclic models fall into mode collapse accompanied by the success of discriminator.
To address this problem, we propose an novel Bayesian cyclic model and an integrated cyclic framework for inter-domain mappings. The proposed method motivated by Bayesian GAN explores the full posteriors of cyclic model via sampling latent variables and optimizes the model with maximum a posteriori (MAP) estimation. Hence, we name it Bayesian CycleGAN.
In addition, original CycleGAN cannot generate diversified results. But it is feasible for Bayesian framework to diversify generated images by replacing restricted latent variables in inference process.
We evaluate the proposed Bayesian CycleGAN on multiple benchmark datasets, including Cityscapes, Maps, and Monet2photo.
The proposed method improve the per-pixel accuracy by 15\% for the Cityscapes semantic segmentation task within origin framework and improve 20\% within the proposed integrated framework, showing better resilience to imbalance confrontation. The diversified results of Monet2Photo style transfer also demonstrate its superiority over original cyclic model.
We provide codes for all of our experiments in \url{https://github.com/ranery/Bayesian-CycleGAN}.
\end{abstract}

\begin{IEEEkeywords}
Generative Adversarial Networks, Bayesian Inference, Adversarial Learning, Domain Translation.
\end{IEEEkeywords}}

\maketitle
\IEEEdisplaynontitleabstractindextext
\IEEEpeerreviewmaketitle

\section{Introduction}
\lettrine[lines=2]{L}{earning} mappings between different domains has recently received increasing attention, especially for image-to-image domain translation tasks, like semantic segmentation and style transfer.
Powered by the representing capabilities of deep neural networks, Isola et al. \cite{isola2017image} leverages Generative Adversarial Networks (GANs) \cite{NIPSGAN} to develop a general framework called Pix2Pix for domain translation, which frees us from hand-engineer mapping functions.
To tackle the scenario where paired data is scared, several cyclic frameworks have been proposed to extend Pix2Pix to unsupervised setting, including methods relied on cycle-consistency based constraints such as CycleGAN \cite{CycleGAN2017,pmlr-v70-kim17a,yi2017dualgan,xgan}, while others relied on weight sharing constraints \cite{liu2016coupled,liu2017unsupervised,huang2018munit}.
In particular, these cyclic frameworks have two critical assumptions:
\begin{itemize}
    \item \textit{Assumption 1:} Dependent GANs accurately model the true data distribution of target domain.
    \item \textit{Assumption 2:} The underlying domain translation is deterministic (One-to-One mapping).
\end{itemize}

\begin{figure}[t]
    \centering
    \includegraphics[width=0.92\linewidth]{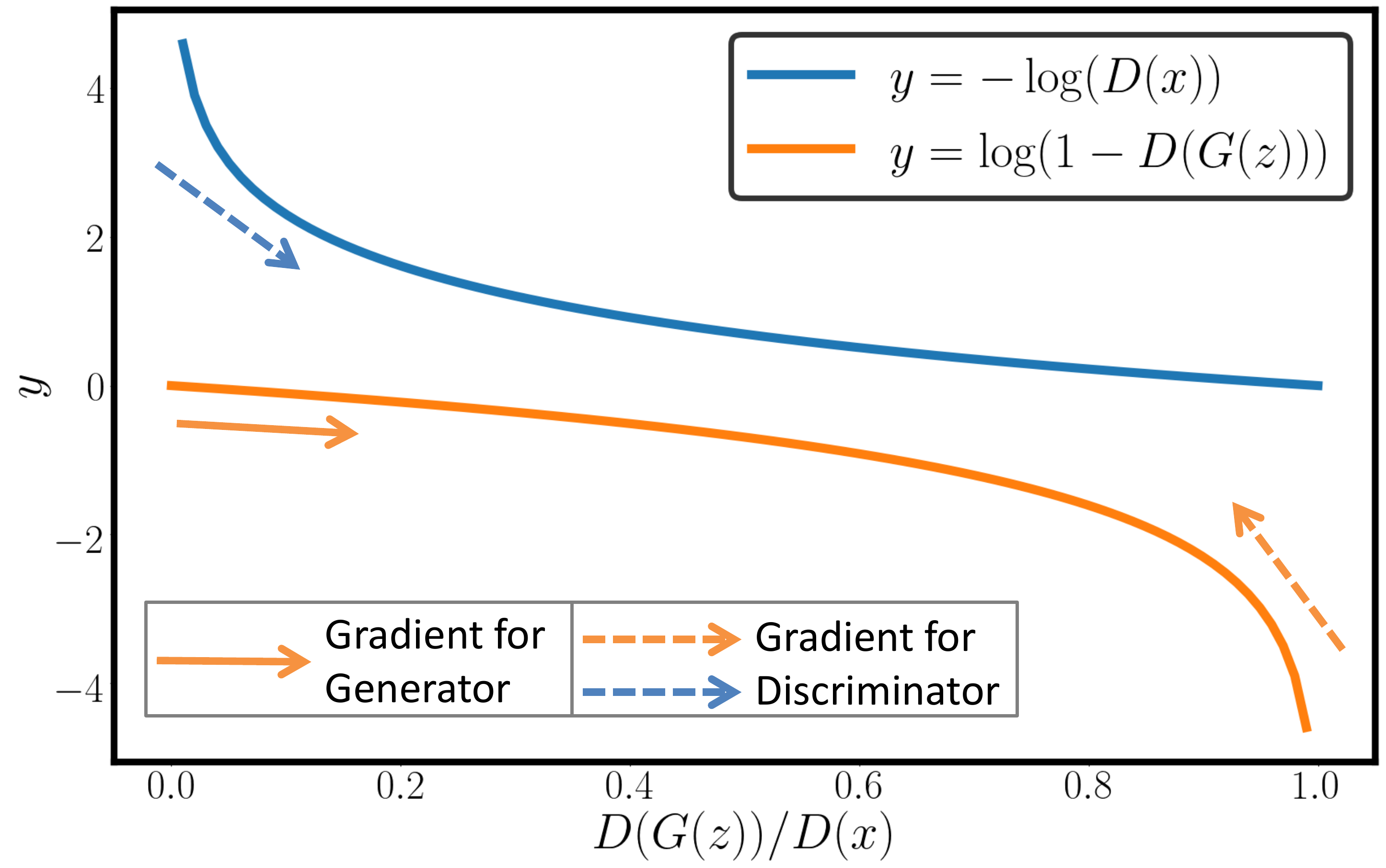}
    \caption{Illustration of GANs objective. Abscissa axis represents auxiliary variable $D(x)$ or $D(G(z))$ in order to match curves.}
    \label{fig:gan_stability}
    \vspace{-2em}
\end{figure}

\textit{Assumption 1} reveals the fundamental goal of GANs.
However, their learning objective can lead to \textit{mode collapse}, where generator simply memorizes a few samples from training set to fool discriminator.
Take the basic objective function (Equ. \ref{gan_objective}) as an example, discriminator $D$ wants to distinguish the real data distribution $x$ from generated fake data distribution $G(z)$, thus its mapping grows towards $D(x) \sim 1$ and $D(G(z)) \sim 0$ to the extent that maximizing objective function. While generator $G$ aims at deceiving discriminator so as to increases $D(G(z))$ and then minimizes objective function\textcolor{blue}{:}
\begin{align}\label{gan_objective}
    \min_{G} \max_{D} & \, \, \mathbb{E}_{x \sim p_{data}(x)} [\log D(x)] \\ \nonumber
    & + \mathbb{E}_{z \sim p_z(z)} [\log(1-D(G(z)))]\textcolor{blue}{.}
\end{align}
To better understand why min-max optimization traps generator into mode collapse, we illustrate the objective function at Fig. \ref{fig:gan_stability}.
When discriminator goes wrong, there are great gradients for updating parameters. While when generated samples are likely fake in the kick-off process of training, gradients in that region are relatively flat, resulting in an initially unfair game and accounting for the mode collapse.
By virtue of the illustration we notice the weak training stability of GANs: unbalanced optimization complexity eases mode collapse,
which is consistent with recently reported obstacles in training GANs \cite{arjovsky2017wasserstein, gulrajani2017improved, mao2017least, saatci2017bayesian, bayesian2019}. All these evidences put the \textit{Assumption 1} in doubt.

To tackle the unsupervised generation tasks, cyclic GANs leverage GANs to learn mutual mappings between different domains using coupled generators without supervision (i.e., no need for paired dataset). Specifically, a cyclic framework consists of a couple of generators: $G_A$ maps one domain to another while $G_B$ maps inversely,
and therefore inherits GANs' unbalanced optimization objectives. Moreover, it remains harder for the generators in cyclic framework to stably model the mappings due to the lack of supervision.
The existing cyclic frameworks achieve expedient balance by elaborately tuning hyper-parameters and using tricks like buffering generated images to facilitate generator training, avoiding the untimely success of discriminator. Otherwise it would be difficult to remain stable, the generator is prone to collapse the variance of data distribution (memorize training samples) to achieve infinite likelihood \cite{saatci2017bayesian}.

To solve the stability issue in GANs, Bayesian GAN \cite{saatci2017bayesian} more accurately models the true data distribution by fully representing the posterior distribution over the parameters of both generators and discriminators, i.e., Bayesian GAN forms posteriors by sampling $\theta_g \!\!\sim\!\! p(\theta_g)$ and $\theta_d \!\!\sim\!\! p(\theta_d)$, where $\!\theta_g\!$ and $\!\theta_d\!$ represent network parameters in generator and discriminator, respectively.
Inspired by the success of Bayesian GAN, we aim to solve the harder stability issue in cyclic framework.
However, there is a non-trivial issue for directly applying Bayesian GAN to the cyclic frameworks.
Specifically, since cyclic framework includes two generators with inverse function, we can not follow Bayesian GAN to form posteriors by sampling generators and discriminators:
\begin{enumerate}
    \item
    First, considering the cycle-consistency loss from $x$:
    $\mathbb{E}_{x \sim p_X(x)} \| G_B(G_A(x)) - x \|_1$,
    if we sample the two generators simultaneously, the reconstructed images will be skimble-scamble practically, i.e., for two $G_A$ samples, $G_{B_1}(G_{A_1}(x)) \neq G_{B_1}(G_{A_2}(x))$. While the cycle-consistency loss in the cyclic framework is dedicated to make the reconstructed images the same as inputs (i.e., $G_{B_1}(G_{A_1}(x))) = G_{B_1}(G_{A_2}(x)) = x$). The only way for the samples of coupled generators to be mutually inverse is if they collapse to being roughly one pair ($G_{A_1} = G_{A_2}$), then the idea of sampling coupled generators fails. Similar failures can also be found in the CycleGAN with stochastic mappings \cite{almahairi2018augmented}.

    \item
    In addition, sampling the coupled generators is very inefficient. It not only increases the memory footprint but also makes the training time/energy quadratically increase. For example, if we sample three $G_A$ and three $G_B$, the memory will be 3$\times$ and the training time will be 9$\times$ as compared to the original framework.
\end{enumerate}
To solve above issues,
we propose a new Bayesian framework with regularized priors called Bayesian CycleGAN.
Instead of sampling generators, we introduce a latent space for exploring the full posteriors.
Specifically, we combine source images with corresponding certain volume of sampled latent variables to construct variant-concatenate inputs, transferring network sampling to input domain sampling in order to avoid cyclic issues mentioned above.
By exploring full posteriors, Bayesian CycleGAN enhances generator training to resist crazy learning discriminator, and therefore alleviates the risk of mode collapse, boosting realistic color generating and multimodal distribution learning.
The proposed Bayesian CycleGAN models the true data distribution more accurately by fully representing the posterior distribution over the parameters of both generator and discriminator as demonstrated by experiments.

To gain more insights of the improved stability,
we increase the number of training samples for two discriminators to form an integrated cyclic framework by additionally introducing reconstructed images as fake samples apart from the generated images, where a balance factor $\gamma$ is used to indicate the coefficient of the added reconstructed image numbers as compared to the generated ones. The balance factor $\gamma$ has two helpful functions:
1) it enhances the reconstructed learning and boosts realistic color distribution; 2) it effects the stability of cyclic framework by accelerating the learning process of discriminator.
We strictly proof the global optimality of proposed integrated cyclic framework in Section \ref{theoretical_analysis}.
The derivation shows that the introduced adversarial loss allows the balance of power even more tilted towards discriminator, allowing us to control the degree of difficulty for stabilizing cyclic framework and then conduct ablation studies of the tolerance of Bayesian CycleGAN in different confrontational contexts, i.e. different values of $\gamma$.
It further reveals the fundamental motivation of being ``Bayesian'': bearing remarkable resilience to the unbalanced game and thus achieving a win-win situation in lieu of the unilateral victory of discriminators\cite{saatci2017bayesian}.

\textit{Assumption 2} renders cyclic framework unable to handle tasks requiring flexible mappings.
Furthermore, current cyclic framework cannot generate diversified results even when equipped with latent variable as documented in \cite{almahairi2018augmented}.
Thinking outside of the box, if we divides the latent space and use only one specific part to fit training process \cite{tnnls-domain},
then it will be feasible to diversify generating in inference process by replacing the specific latent variables with others.
In practice, we adopt an encoder to generate that specific part from source domains, which we term as the \emph{statistic feature map (SFM)}.
We show the diversified effect of this delineative idea with Monet2photo style transfer experiment in Section \ref{experiment}.

To the best of our knowledge, we present the first detailed treatment of Bayesian CycleGAN, including the novel Bayesian generalization, sampling based inference, and semi-supervised learning extension.
One can also construct Bayesian cyclic framework in other ways, like using dropout as a Bayesian approximation. We compare with the dropout version in experiments.
The proposed Bayesian CycleGAN shows significant improvement for unsupervised semantic segmentation, improving 15\% for per-pixel accuracy and 4\%  for class intersection of union (IOU) compared to CycleGAN \cite{CycleGAN2017}.

The main contributions of our work are summarized as:
\begin{enumerate}
\item We propose a novel cyclic model called \textit{Bayesian CycleGAN} for stabilizing unsupervised learning by exploring the full posteriors via latent sampling.
\item We introduce a hyper-parameter for better balancing the reconstructed learning and training stability, and prove this variation also has a global optimum.
\item We impose restriction on the latent space for generating diversified images by adding encoder networks.
\item We perform experiments on several representative data sets, testifying that the proposed model outperforms the original model and confers a defense against instability.
\end{enumerate}
The rest of this paper is organized as follows.
Section \ref{related_works} recaps recent related works and elaborates the difference between them and proposed method.
Section \ref{proposed_method} sets up the translation problem and proposed integrated cyclic framework, and then formulates the full posteriors of Bayesian CycleGAN.
Section \ref{theoretical_analysis} presents the shift global optimum and equilibrium conditions of proposed method.
Section \ref{experiment} shows plenty of experiments and discussions for un/semi-supervised learning.
Finally, Section \ref{conclusion} concludes this paper.

\vspace{-0.5em}
\section{Related Works}
\label{related_works}
\textbf{Image-to-Image Translation}
Image-to-Image translation is a classical problem in computer vision.
Motivated by the success of generative model like Generative Adversarial Networks \cite{NIPSGAN}, our focus has begun to change from one selected artwork style transfer to inter-domain translation.
For supervised learning, conditional GANs has been used in Pix2pix \cite{isola2017image} to convert images from one domain to another, like landscapes to semantic labels. Also, BicycleGAN \cite{NIPS2017_6650} is the enhanced version of Pix2pix since it generates diversified outputs. Besides, VAE \cite{kingma2013auto,tnnls-denoising} is a probabilistic model with an encoder to map source images to a latent representation and a decoder to reconstruct back to source domain, which can also be used in inter-domain translation. We adopt a VAE-like network to generate the restricted latent sampling, but there are two main differences: 1) our network replaces the encoder/decoder to down/up-sample layers; 2) we abolish the $l_1$ loss between reconstructed image and source image.
For unsupervised learning, there has been a recent surge on cyclic GANs \cite{CycleGAN2017,liu2016coupled,liu2017unsupervised,yi2017dualgan,taigman2016unsupervised,tnnls-adversarial,ouyang2018pedestrian,Li2020CVPR,seqattngan}, which aim to learn a joint distribution given marginal distribution in each individual domain.
Zhu \textit{et al.} present CycleGAN to learn mappings without paired data, achieving fantastic results \cite{CycleGAN2017}.
Almahairi \textit{et al.} extend it to an augmented many-to-many mapping method, taking Gaussian noise into another cycle with GAN loss and $l_1$ loss restrained and training normal cycle along with the latent cycle.

\textbf{GANs' Stability Exploration}
Many recent works have focused on improving the stability of GANs training \cite{nowozin2016f,mao2017least,arjovsky2017wasserstein,li2017mmd,gulrajani2017improved,zhai2016generative}, such as methods using Wasserstein distance \cite{arjovsky2017wasserstein}, latent feature matching \cite{gulrajani2017improved,mroueh2018sobolev}, MMD loss \cite{li2017mmd},
architecture search \cite{Gong2019ICCV}
along with Bayesian methods \cite{probgan2019,saatci2017bayesian, bayesian2019}.
Among which Bayesian GAN is the most related work to our proposed method, it explores an expressive posterior over the parameter of generators for avoiding \textit{mode collapse}.
Sampling over the generator can be formalized as follows:
\begin{enumerate}
    \item Sample from posteriors $\theta_g \sim p(\theta_g | \theta_d)$;
    \item Sample data points $z^{(1)}, \cdots, z^{(n)} \sim p(z)$;
    \item Inference $\tilde{x}^{(j)} = G(z^{(j)}; \theta_g) \sim p_{\tilde{x}}(x)$.
\end{enumerate}
Having such samples can be very useful, the samples for $\theta_g$ help to form a committee of generators to strengthen the discriminator, and samples for $\theta_d$ in turn amplifies the overall adversarial signals, thereby further improving the learning process and alleviating \textit{mode collapse}. In practice, they use Stochastic Gradient Hamiltonian Monte Carlo (SGHMC) \cite{chen2014stochastic} for posterior sampling.

\underline{\textit{Difference with Bayesian GAN}}
There are two main differences: 1) instead of generating predictive distribution with stochastic variable, our posteriors are based on cyclic framework, mapping between different domains; 2) Bayesian GAN explores whole posterior distribution over the network weights via performing network sampling in iterative process while our model adopts iterative MAP optimization along with latent sampling due to the practical nature of cyclic model.

\underline{\textit{Difference with CycleGAN}}
There are three main differences:
1) Bayesian CycleGAN proposes the Bayesian formulation with Gaussian priori for original cyclic framework in theory;
2) it is optimized with MAP and latent sampling, bringing robustness improvement via the inductive bias, while CycleGAN optimization can be viewed as the maximum likelihood estimation (MLE);
3) in addition to one-to-one deterministic mapping, Bayesian CycleGAN can generate diversified outputs by imposing restriction on the latent variable.

\vspace{-0.5em}
\section{Proposed Method}
\label{proposed_method}

In this section, we first provide the background of cyclic GANs.
Next, targeting the consequent problems (weak training stability and lack of diversity) resulting from the two underlying assumptions of cyclic GANs, we propose: 1) Bayesian formulation of integrated cyclic framework, in which we introduce balance factor $\gamma$ for the ablation study of training stability;
2) formulation and objective function of Bayesian CycleGAN;
3) VAE-like encoder to produce SFM and leverage it for diversified generating.

\vspace{-1em}
\subsection{Problem Setup and Cyclic framework}

Given two domains $X$ and $Y$ with associated datasets $\mathcal{X} = \{x^{(i)}\}$ and $\mathcal{Y} = \{y^{(j)}\}$,
the goal is to learn the mappings between them using unpaired samples from marginal distribution $p_{X}(x)$ and $p_{Y}(y)$ in each domain.

\begin{figure*}
\centering
\includegraphics[width=\linewidth]{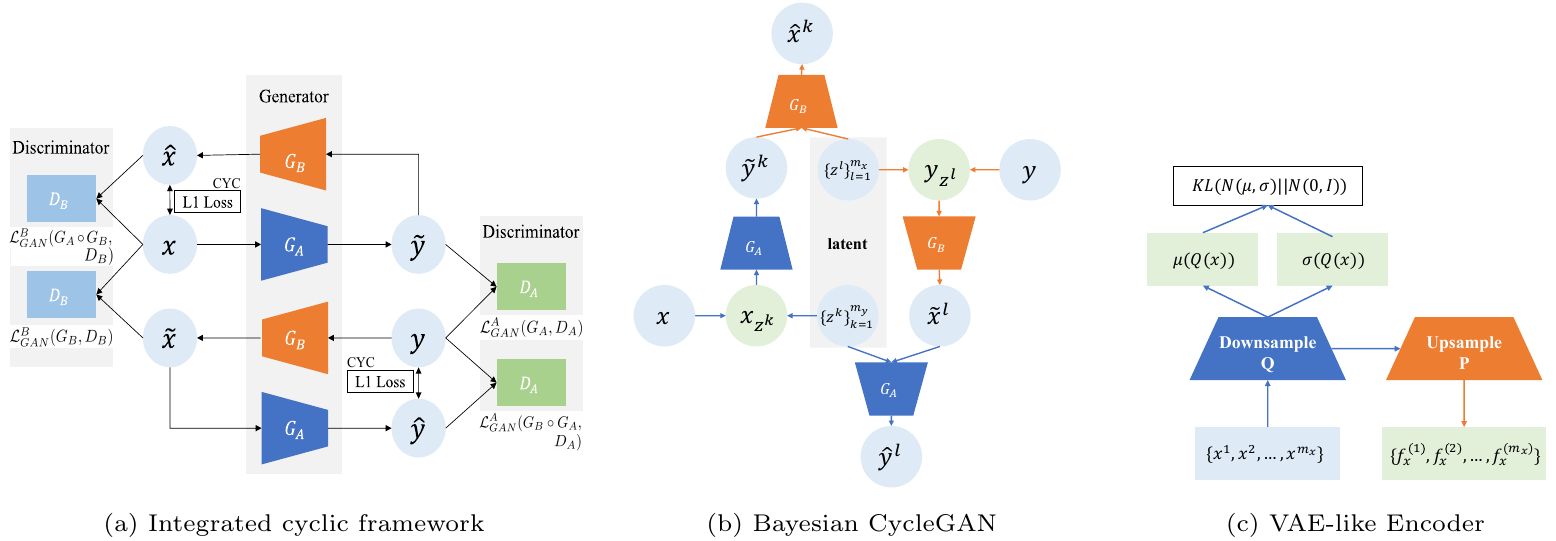}
\caption{
(a) Integrated cyclic framework: We explore the integrated cyclic framework by adding GAN loss between reconstructed images and source images.
(b) Bayesian CycleGAN: We deduce the full posteriors over parameters and adopt iterative MAP optimization along with the latent sampling to update networks. (c) VAE-like Encoder: We use this encoder to generate SFM latent variable which can be replaced in the inference process to get diversified outputs.}
\label{model}
\end{figure*}

\noindent
\textbf{Cyclic GANs} To achieve above goal, cyclic GANs are proposed to leverage GANs to represent corresponding mappings. They try to estimate the conditional distributions $p(y|x)$ and $p(x|y)$ based on samples from marginals. Specifically, cyclic GANs use two generators for estimating the conditions: $G_A(x; \theta_{ga}): X \rightarrow Y$ and $G_B(y; \theta_{gb}): Y \rightarrow X$.
Training these generators for accurately modeling the true data distribution requires cyclic framework to satisfy two constraints:
\begin{itemize}
    \item \underline{\textit{Marginal matching}}: the output samples of each mapping function (generator) should match target domain under the supervise of discriminator:
    \begin{align}
        \mathcal{L}_{GAN}^A(G_A, D_A) & = \mathbb{E}_{y \sim p_{Y}(y)} [\log D_A(y)] \\ \nonumber
        & + \, \mathbb{E}_{x \sim p_X(x)} [\log(1-D_A(G_A(x)))].
    \end{align}
    A similar adversarial loss $\mathcal{L}_{GAN}^{B}(G_B, D_B)$ is defined for marginal matching in reverse direction.
    \item \underline{\textit{Cycle-consistency}}: mapping the generated samples back to reconstructed end should produce images close to marginal distribution in source domain under the control of cycle-consistency loss:
    \begin{align}
        \mathcal{L}_{CYC}^A(G_A, G_B) = \mathbb{E}_{x \sim p_{X}(x)} \| G_B(G_A(x)) - x \|_1.
    \end{align}
    Similar cycle-consistency loss $\mathcal{L}_{CYC}^B(G_B, G_A)$ for another inverse cycle $y \rightarrow x \rightarrow y$.
\end{itemize}
The total loss function of cyclic GANs is:
\begin{align}
    \mathcal{L}(G_A&, G_B, D_A, D_B) \\ \nonumber
    &= \mathcal{L}_{GAN}^A(G_A, D_A) + \mathcal{L}_{GAN}^B(G_B, D_B) \\ \nonumber
    & + \lambda (\mathcal{L}_{CYC}^A(G_A, G_B) + \mathcal{L}_{CYC}^B(G_B, G_A)),
\end{align}
where $\lambda$ controls the relative importance of the two constraints. We aim to solve:
\vspace{-0.5em}
\begin{align}
    G_A^*, G_B^* = \text{arg} \min_{G_A, G_B} \max_{D_A, D_B} \mathcal{L}(G_A, G_B, D_A, D_B)
    \vspace{-0.5em}
\end{align}
by iteratively updating parameters of generators and discriminators using gradient back-propagation.

\noindent\textbf{Integrated Cyclic Framework}
The difference with original cyclic framework is that the integrated cyclic framework considers another cycle matching constraint:
\begin{align}
    \mathcal{L}_{GAN}^{A}(G_B \circ G_A, D_A) & = \mathbb{E}_{y \sim p_{Y}(y)} [\log D_A(y)] \\ \nonumber
        & + \, \mathbb{E}_{y \sim p_{Y}(y)} [\log(1\!-\!D_A(G_A(G_B(y))))],
\end{align}
which represents adversarial loss between reconstructed images and source images. A similar loss $\mathcal{L}_{GAN}^{B}(G_A \circ G_B, D_B)$ can be defined inversely as shown in Fig. \ref{model} (a).
Next we analyze from the judgemental view of discriminators.
For the sake of clarity, we denote $G_{A}(x; \theta_{ga})$ as $\tilde{y}$ and $G_{A}(G_{B}(y; \theta_{gb}); \theta_{ga})$ as $\hat{y}$, similar notations are used for $x$, $\tilde{x}$ and $\hat{x}$.
E.g., for discriminator $D_A$, the three parts of distribution $p_Y(y)$ have specific forms: $(y, t=1)$, $(\tilde{y}, t=0)$ and $(\hat{y}, t=0)$, in which $t$ is the label for discrimination, $t=1$ represents true data distribution while $t=0$ indicates generated distribution. The conditional likelihoods are then defined as:
\begin{align}
p( t=1 | X, Y, Z_y=y, \theta_{da}, \theta_{g}) &\propto D_A(y;\theta_{da}), \\
p( t=0 | X, Y, Z_y=\tilde{y}, \theta_{da}, \theta_{g}) &\propto  1 \! - \! D_A(\tilde{y}; \theta_{da}), \\
p( t=0 | X, Y, Z_y=\hat{y}, \theta_{da}, \theta_{g}) &\propto  1 \! - \! D_A(\hat{y};\theta_{da}).
\end{align}
The auxiliary variable is noted as $Z_y\in\{y, \tilde{y}, \hat{y}\}$. Here $Z_y$ is a multinomial
distribution for indicating the input assignments. Assume the ratio between number of pairs $|\{(y, t)\}| : |\{(\tilde{y}, t)\}| : |\{(\hat{y}, t)\}|=1+\gamma:1:\gamma$, where $\gamma$ is the introduced balance factor, indicating the coefficient of adding adversarial loss in reconstructed ends. The balance factor can be viewed as zero in original cyclic framework since there is only $l_1$ loss constraint in reconstructed ends. We prove that this variation also has a global optimum theoretically in Section \ref{theoretical_analysis}, and show that the balance factor can adjust model stability as well as enhance reconstructed learning performance in Section \ref{experiment}.

\subsection{Formulation of Bayesian CycleGAN}
\noindent\textbf{Posteriors Derivation}
The goal of the whole posterior approximation is to make generators broad and highly multimodal, each mode in the posterior over the network weights could correspond to wildly different generators, each with their own meaningful interpretations \cite{saatci2017bayesian}.
The posterior approximation therefore can be leveraged to solve the aforementioned stability issue in cyclic framework (i.e., the inherited  GANs' unbalanced objectives coupled with the lack of supervision).
We propose Bayesian CycleGAN to alleviate mode collapse by fully representing
posterior distribution of both the generators and discriminators. Fig. \ref{model} (b) shows the inference process of Bayesian CycleGAN.
Instead of sampling generators, we introduce latent space and regularized priors to formulating full posteriors.
The latent space consists of various latent variables, we discuss two kinds in this paper, here we use random noise $z$ and later we will introduce another kind of latent variables.
Specifically, we combine source images with sampled latent variables to construct variant-concatenate inputs, e.g., combine $x \in \mathbb{R}^{C \times H \times W}$ and fixed volume of sampled latent Gaussian random variable $z^{(j)} \in \mathbb{R}^{1 \times H \times W}, j=1,2 \cdots m_x$ to construct a set of input $x_{z^{(j)}} \in \mathbb{R}^{(C+1) \times H \times W}$,
in which $C$ means number of channels, $H$ and $W$ represent input spatial dimensions, $m_x$ stands for the volume of latent variables.
To articulate the posteriors of integrated cyclic framework,
we infer different sampled space that posteriors conditioned on :
$\Omega_{g} = (x, y, x_z, y_z)$,
$\Omega_{da} = (x, y, Y, x_z, y_z)$, and
$\Omega_{db} = (x, y, X, x_z, y_z)$, in which $x, y$ indicate normal input variables, $X, Y$ represent source domains while $x_z, y_z$ mean variant-concatenated inputs.
Therefore,
the posterior of discriminator $D_A$ w.r.t each sampled $x_z$
can be derived as Equ. (\ref{da_eg}). We implicitly express likelihood function and use $\propto$ for convenience since $\log$ function is monotonically increasing.
\begin{equation}\label{da_eg}
\begin{split}
& p(\theta_{da} | \Omega_{da}, t, \theta_{ga}, \theta_{gb}) \\
& \propto (D_A(y;\theta_{da}))^{(1+\gamma)} \times (1 - D_A(\hat{y};\theta_{da}))^{\gamma} \\
& \times (1 - D_A(G_A(x_z);\theta_{ga});\theta_{da}) \times p(\theta_{da} | \alpha_{da}).
\end{split}
\end{equation}
$p(\theta_{da} | \alpha_{da})$ is regularized prior over discriminator $D_A$ in which $\alpha$ means hyper-parameters for $D_A$.
Other posteriors for $\theta_{db}$ and $\theta_{g}$
(parameters of generator)
can be derived analogously.

\noindent\textbf{Bayesian Marginalization}
To get rid of $\Omega$ which is defined on the training dataset, we employ Bayesian marginalization on the derived posterior distribution of discriminator $D_A$, which allows the estimated network parameters to generalize to the inference on validation/testing dataset without distortion theoretically after fully iterative MAP (Maximum A Posteriori) training.
Here we use MAP ranther than MLE (Maximum Likelihood Estimation). Though they are both methods for estimating the distribution of target variables given input data and the model,
MLE suffers from the problem that it overfits the data while MAP sovles this problem by adding priori regularization \cite{rynkiewicz2012general,vehtari2000bayesian}. We also show the difference in Equ. (\ref{MLE_MAP}), assume we have likelihood function $P(X|\theta)$, then MLE/MAP aims to estimate the parameter $\theta$ by maximizing corresponding objectives:
\begin{equation}\label{MLE_MAP}
\begin{split}
    \theta_{MLE} &= \text{arg} \max_{\theta} P(X|\theta) = \text{arg} \max_{\theta} \prod_i P(x_i|\theta) \\
    \theta_{MAP} &= \text{arg} \max_{\theta} P(X|\theta) P(\theta) \\
    & = \text{arg} \max_{\theta} \prod_i P(x_i|\theta) + \log P(\theta).
\end{split}
\end{equation}
Build on the above background information, we next formulate the marginalization process, taking the posterior estimation of discriminator $D_A$: $p(\theta_{da} | \Omega_{da}, t, \theta_{ga}, \theta_{gb})$ as an example:
\begin{equation}\label{marginalization}
\begin{split}
   & p(\theta_{da} | t, \theta_{ga}, \theta_{gb}) \\
   & = \! \int \! p(\theta_{da} , \Omega_{da} | t, \theta_{ga}, \theta_{gb}) \! \,d\Omega_{da}  \\
    \! & = \! \int \! p(\theta_{da} | \Omega_{da}, t, \theta_{ga}, \theta_{gb}) \! \cdot \! p(\Omega_{da} | \theta_{ga}, \theta_{gb}) \! \,d\Omega_{da} \\
    \! & = \! \int \!\! p(\theta_{da} | \Omega_{da}, t, \theta_{g}) \! \cdot \! p(x, y, x_z, y_z) \! \cdot \! p(Y | x, y, x_z, y_z)  d\Omega_{da} \\
    \! & = \sum_{i=1}^{n_x} \sum_{j=1}^{n_y} \sum_{k=1}^{m_y} \sum_{l=1}^{m_x} p(\theta_{da} | \Omega_{da}, t, \theta_{ga}, \theta_{gb}), \nonumber
\end{split}
\end{equation}
in which $n_x, n_y$ represent the batch size for training process; $m_x, m_y$ indicate fixed latent sampling volume.
Built upon over above marginalization, the posterior for discriminator $D_A$ can be rewritten as Equ. (\ref{da}).
Note the continuous addition has been converted to continuous multiplication since we implicitly ignore $\log$ function by using $\propto$.
\begin{equation}\label{da}
\begin{split}
  & p(\theta_{da} | t, \theta_{ga}, \theta_{gb}) \\
  & \! \propto \! \prod_{j=1}^{n_y} \! (D_{A}(y^{(j)}; \theta_{da}))^{(1+\gamma)m_{x}} \! \times \! \prod_{j=1}^{n_y} \prod_{l=1}^{m_{x}} \! (1 \! - \! D_{A}(\hat{y}^{(j,l)}; \theta_{da}))^{\gamma} \\
  & \times \prod_{i=1}^{n_x} \prod_{k=1}^{m_{y}} (1-D_{A}(G_{A}(x_{z}^{(i,k)}; \theta_{ga}); \theta_{da})) \times p(\theta_{da} | \alpha_{da}).
\end{split}
\end{equation}
Then the embedded loss function for discriminator $D_A$ can be formulated as:
\begin{sequation}
\begin{split}
    & \mathcal{L}_{GAN}(D_A) \! = \! (1+\gamma) m_x \mathbb{E}_{y \sim p_{y}} [\sum_{j=1}^{n_y}\! \log D_A(y^{(j)})]\! + \! \| \theta_{da} \|_1^2 \!+ \!\mathbb{E}_{y \sim p_{\tilde{y}}} \\
    & [\sum_{i=1}^{n_x}\!\sum_{k=1}^{m_y} \! \log (1\!-\!D_A(\tilde{y}^{(i,k)}))] \! + \! \gamma \mathbb{E}_{y \sim p_{\hat{y}}} [\sum_{j=1}^{n_y}\!\sum_{l=1}^{m_x} \! \log (1\!-\!D_A(\hat{y}^{(j,l)}))].
\end{split}
\end{sequation}
Similarly, we can derive the posterior for $\theta_{db}$ as Equ. (\ref{db}).
\begin{equation}\label{db}
\begin{split}
   & p(\theta_{db} | t, \theta_{ga}, \theta_{gb}) \\
   & \! \propto \! \prod_{i=1}^{n_x} (D_{B}(x^{(i)}; \theta_{db}))^{(1+\gamma)m_{y}} \! \times \! \prod_{i=1}^{n_x} \prod_{k=1}^{m_{y}} \! (1 \! - \! D_{B}(\hat{x}^{(i,k)}; \theta_{db}))^{\gamma} \\
   & \times \prod_{j=1}^{n_y} \prod_{l=1}^{m_{x}} (1-D_{B}(G_{B}(y_{z}^{(j,l)}; \theta_{gb}); \theta_{db})) \times p(\theta_{db} | \alpha_{db}).
\end{split}
\end{equation}
Then the loss function for discriminator $D_B$ is:
\begin{sequation}
\begin{split}
    & \mathcal{L}_{GAN}(D_B) \! = \! (1+\gamma) m_y \mathbb{E}_{y \sim p_{y}} [\sum_{i=1}^{n_x}\! \log D_B(y^{(i)})]\! + \! \| \theta_{db} \|_1^2 \!+ \!\mathbb{E}_{y \sim p_{\tilde{y}}} \\
    & [\sum_{j=1}^{n_y}\!\sum_{l=1}^{m_x} \! \log (1\!-\!D_B(\tilde{y}^{(j,l)}))] \! + \! \gamma \mathbb{E}_{y \sim p_{\hat{y}}} [\sum_{i=1}^{n_x}\!\sum_{k=1}^{m_y} \! \log (1\!-\!D_B(\hat{y}^{(i,k)}))].
\end{split}
\end{sequation}
For the posterior estimation of generator, we generate estimated samples without needing to consider discriminative label $t$.
Based on the same conditional likelihood form defined above,
we are then able to formulate the posterior of generator $G_A$ and $G_B$, taking both marginal matching and cycle consistency constraints into consideration:
\begin{equation}\label{ga-gb}
\begin{split}
  & p(\theta_{ga}, \theta_{gb} | \theta_{da}, \theta_{db}) \\
  & \!\! \propto \!\! \left(\prod_{i=1}^{n_x} \prod_{k=1}^{m_{y}} \! D_A( G_A(x_{z}^{(i,k)})) \!\! \times \! (D_B(\hat{x}^{(i,k)}))^{\gamma} \!\! \cdot \! e^{- \lambda \left\| \hat{x}^{(i,k)} \! - \! x^{(i)} \right\|} \! \right) \\
  & \!\! \times \!\! \, \left(\prod_{j=1}^{n_y} \prod_{l=1}^{m_{x}} \! D_B( G_B(y_{z}^{(j,l)})) \! \times \! (D_A(\hat{y}^{(j,l)}))^{\gamma} \! \cdot \! e^{- \lambda \left\| \hat{y}^{(j,l)} \! - \! y^{(i)} \right\|} \! \right) \\
  & \times p(\theta_{ga} | \alpha_{ga})  \times p(\theta_{gb} | \alpha_{gb}).
\end{split}
\end{equation}
Corresponding Loss function for cascade Generator $G_A, G_B$ can then be formulated as:
\begin{sequation}
\begin{split}
    & \mathcal{L}_{GAN}(G) = \mathbb{E}_{x \sim p_{x}} [\sum_{i=1}^{n_x}\!\sum_{k=1}^{m_y} \! \log \! D_A(x^{(i,k)})] + \| \theta_{ga} \|_1^2 + \| \theta_{gb} \|_1^2 + \\
    & \mathbb{E}_{x\sim p_{\hat{x}}} [\sum_{i=1}^{n_x}\!\sum_{k=1}^{m_y} \log D_B(\hat{x}^{(i,k)})] + \mathbb{E}_{y \sim p_{y}} [\sum_{j=1}^{n_y}\!\sum_{l=1}^{m_x} \log D_B(y^{(j,l)})] + \\
    & \mathbb{E}_{y \sim p_{\hat{y}}} [\sum_{j=1}^{n_y}\sum_{l=1}^{m_x} \log D_A(\hat{y}^{(j,l)})] + \lambda (\| \hat{x}^{(i,k)} \! - \! x^{(i)} \| \! + \! \|\hat{y}^{(j,l)} \! - \! y^{(i)} \|).
\end{split}
\end{sequation}
\noindent\textbf{Posteriors Derivation of LS-GAN}
To reduce model oscillation
, we follow the original CycleGAN to replace the negative log likelihood loss function with a least square adversarial loss, shown in Equ. (\ref{ls-objective}).
\begin{equation}\label{ls-objective}
    \begin{split}
        \min_D & \frac{1}{2} \mathbb{E}_{x \sim p_{data}(x)}[(D(x) - 1)^2] + \frac{1}{2} \mathbb{E}_{z \sim p_z(z)}[D(G(z))^2] \\
        \min_G & \frac{1}{2} \mathbb{E}_{z \sim p_z(z)}[(D(G(z)) - 1)^2].
    \end{split}
\end{equation}
We formulate the alternative posteriors by applying the combination of $l_1$ loss and least square adversarial loss to proposed Bayesian CycleGAN. The corresponding derivations, including integrated cyclic framework analysis, Bayesian marginalization and posterior deduction, remain same as above formulation based on standard adversarial loss. The detailed estimation of posteriors based on $l_1$-least square adversarial loss are derived as the following Equ. (\ref{da-ls}),(\ref{db-ls}),(\ref{ga-gb-ls}).

\begin{equation}\label{da-ls}
\begin{split}
   & p(\theta_{da} | t, \theta_{ga}, \theta_{gb}) \\
   & \!\! \propto \!\! \prod_{j=1}^{n_y} \! e^{-(1+\gamma) \cdot m_x \cdot (D_{A}(y^{(j)}; \theta_{da}) - 1)^{2}} \!\! \times \!\! \prod_{j=1}^{n_y} \prod_{l=1}^{m_{x}} \! e^{-\gamma \cdot (D_{A}( \hat{y}^{(j,l)}; \theta_{da}))^{2}} \\
   & \!\! \times \! \prod_{i=1}^{n_x} \prod_{k=1}^{m_y} e^{-(D_{A}( G_{A}(x_{z}^{(i,k)}; \theta_{ga}); \theta_{da}))^{2}} \times p(\theta_{da} | \alpha_{da}).
\end{split}
\end{equation}

\begin{equation}\label{db-ls}
\begin{split}
    & p(\theta_{db} | t, \theta_{ga}, \theta_{gb}) \\
    & \!\! \propto \!\! \prod_{i=1}^{n_x} \! e^{-(1+\gamma) \cdot m_y \cdot (D_{B}(x^{(i)}; \theta_{db}) - 1)^{2}} \!\! \times \!\! \prod_{i=1}^{n_x} \prod_{k=1}^{m_y} e^{-\gamma \cdot (D_{B}( \hat{x}^{(i,k)}; \theta_{db}))^{2}} \\
    & \!\! \times \! \prod_{j=1}^{n_y} \prod_{l=1}^{m_{x}} e^{-(D_{B}( G_{B}(y_{z}^{(j,l)}; \theta_{gb}); \theta_{db}))^{2}} \times p(\theta_{db} | \alpha_{db}).
\end{split}
\end{equation}

\begin{equation}\label{ga-gb-ls}
\begin{split}
   & p(\theta_{ga}, \theta_{gb} | \theta_{da}, \theta_{db}) \\
   & \!\! \propto \!\! \prod_{i=1}^{n_x} \prod_{k=1}^{m_y} \!\! e^{-(D_{A}( G_{A}(x_{z}^{(i,k)})) - 1)^{2} - \gamma (D_{B}( \hat{x}^{(i,k)}) - 1)^{2} - \lambda \left\| \hat{x}^{(i,k)} - x^{(i)} \right\|} \\
     & \!\! \times \! \prod_{j=1}^{n_y} \prod_{l=1}^{m_x} \! e^{-(D_{B}( G_{B}(y_{z}^{(j,l)})) - 1)^{2} - \gamma (D_{A}( \hat{y}^{(j,l)}) - 1)^{2} - \lambda \left\| \hat{y}^{(j,l)} - y^{(i)} \right\|} \\
     & \times p(\theta_{ga} | \alpha_{ga}) \times p(\theta_{gb} | \alpha_{gb}).
\end{split}
\end{equation}
Corresponding loss functions for each network can also be derived by applying $\log$ to each side. In practice we adopt \textbf{LS-GAN loss} and iterative MAP to update their posteriors.

\subsection{Restricted Latent Sampling}
\label{restricted_latent_sampling}
To break the curse of being deterministic mapping in current cyclic frameworks (\textit{Assumption 2}), we propose a restricted latent sampling method in formulated Bayesian CycleGAN scheme to overcome the fundamental flaw of cyclic frameworks: the cycle-consistency encourages the mapping to ignore the latent variable \cite{almahairi2018augmented}.
Specifically, we use only one restricted part of latent space (denoted as \emph{statistic feature map (SFM)}) to fit training, and therefore remaining latent distributions allow us to generate diversified images during inference.
The proposed restricted latent sampling method serves as two main functions
: 1) achieve diversified generating; 2) help to stabilize training process of Bayesian CycleGAN.

Next we will introduce how to generate SFM.
As mentioned before, we use two kinds of latent variables in this paper to constitute latent space, respectively, and have clarified the first kind of random noise. In the experiments we also adopt a VAE-like encoder to generate the SFM $f_x^{(j)} \sim E(x)$ and $f_y^{(j)} \sim E(y)$ to serve as another kind of latent variables, which also accounts for a fraction of overall latent space (relatively restricted as compared to the general random noise).
The VAE-like network $E$ is illustrated in Fig. \ref{model} (c),
which has two differences with variational auto-encoder (VAE) network: 1) $E$ uses down-sample layers $Q$ and up-sample layers $P$ to generate SFM while VAE network adopts encoder and decoder to recover source domain distribution; 2) $E$ is constrainted by $\mathcal{KL}$ loss lying in the bottleneck layer between $Q$ and $P$, which encourages outputs close to Gaussian distribution and thus ensure the generated SFM inferior to source inputs, while VAE is under the constraint of L1 loss between inputs and outputs, making the output distribution close to source domain distribution and therefore dominate the whole latent space.

Similar with previous Bayesian formulation of using Gaussian latent sampling, we keep the basic format of posteriors and derived loss functions. The only difference is to replace the random latent variable $z$ with SFM ($f_{x/y}$).
Specifically, concatenate real source input $x$ with restricted latent variable $f_y$ while concatenate source input $y$ with $f_x$ to construct variant-concatenated input $x_f, y_f$.
In practice we use 3 SFMs for sampling posteriors to achieve trade-off because it already stabilizes training process, more SFMs slightly favor the quality of generated images while incurring extra overhead.
we show: 1) the generated diversified samples using restricted latent sampling method in Monet2Photo experiment of Section \ref{experiment}; 2) the improved stability using restricted latent space to conduct Bayesian CycleGAN in Section \ref{experiment}-C.

\vspace{-1em}
\subsection{Algorithm for Bayesian CycleGAN}
We marginalize the posteriors over generators and discriminators shown in Equ. (\ref{da}), (\ref{db}) and (\ref{ga-gb}) for standard GANs objective based integrated framework or Equ. (\ref{da-ls}), (\ref{db-ls}) and (\ref{ga-gb-ls}) for least square GANs objective based integrated framework. All of deduced posteriors are responsible for optimization and accompanied by a regularization item as prior $p(\theta | \alpha)$. The iterative optimization process illustrates the confrontation between discriminators and generators. Detailed procedures for each iteration are shown in Algorithm \ref{alg:AdamMC}.

\begin{algorithm}[!htbp]
\caption{MAP with Monte Carlo latent sampling for training Bayesian CycleGAN. This is one iteration of updating parameters of generators and discriminators.}
\begin{algorithmic}
\label{alg:AdamMC}
\STATE{$\bullet$ Get parameters $\{\theta_{ea}\}$, $\{\theta_{eb}\}$, $\{\theta_{ga}\}$, $\{\theta_{gb}\}$, $\{\theta_{da}\}$ and $\{\theta_{db}\}$from previous iteration.}
\FOR{number of mini-batch $n$ in one iteration}
{
    \STATE{$\bullet$ Extract $m_x$ samples $x^{(1)}, \dots, x^{(m_x)}$ from $p_x(x)$, each $x^{l}$ has $n$ samples. Use encoder $E_A$ to get SFM $\{f_x^{l}\}_{l=1}^{m_x}$. Then, combine $f_x$ with batch $y^{j}$ to get $y_f^{(j)}$.}
    \STATE{$\bullet$ Extract $m_y$ samples $y^{(k)}, \dots, y^{(m_y)}$ from $p_y(y)$, each $y^{k}$ has $n$ samples. Use encoder $E_B$ to get SFM $\{f_y^{k}\}_{k=1}^{m_y}$. Then, combine $f_y$ with batch $x^{i}$ to get $x_f^{(i)}$.}
    \STATE{$\bullet$ Generate $\tilde{x}$ and $\tilde{y}$ by use of $G_A$ and $G_B$, which take $x_f^{(i)}$ and $y_f^{(j)}$ as input variables. Then combine the fake outputs with SFM $f_x^{m_x}$ and $f_y^{m_y}$ extracted last steps as input to reconstruct cyclic results $\hat{x}^{(i,k)}$ and $\hat{y}^{(j,l)}$ with marginal distribution.}
     \STATE{$\bullet$ Update generators by ascending its posteriors \! $p(\theta_{ga} , \theta_{gb} | \theta_{da}, \theta_{db})$
        \begin{align*}
            \nabla_{\theta_{g}, \theta_{e}} \! \left( \! \frac{\partial \! \log p(\theta_{ga}, \theta_{gb} | \theta_{da}, \theta_{db})}{\partial \theta_{g}} \! + \! \frac{\partial \lambda_{KL} \mathcal{L}_{KL}}{\partial \theta_{e}} \! \right).
        \end{align*} }
    \STATE{$\bullet$ Update discriminator by ascending its posteriors $p(\theta_{da} | t, \theta_{ga}, \theta_{gb})$
        \begin{align} \nabla_{\theta_{da}}\left(\frac{\partial \log p(\theta_{da} | t, \theta_{ga}, \theta_{gb})}{\partial \theta_{da}}\right). \nonumber
        \end{align}}
    \STATE{$\bullet$ Update discriminator by ascending its posteriors $p(\theta_{db} | t, \theta_{ga}, \theta_{gb})$
        \begin{align} \nabla_{\theta_{db}}\left(\frac{\partial \log p(\theta_{db} | t, \theta_{ga}, \theta_{gb})}{\partial \theta_{db}}\right). \nonumber
        \end{align}}
}
\ENDFOR
\end{algorithmic}
\end{algorithm}

\section{Experiments and Evaluations}
\label{experiment}
We first describe the used benchmark datasets and implementation details. Then we perform comparisons among original CycleGAN, Bayesian CycleGAN and several variants, and then analyze how they support our claims and theorems, respectively. In addition, we show the enhanced reconstructed learning effect benefiting from the introduction of balance factor, and the diversified effect resulting from replacing the restricted latent variables during inference process.

\subsection{Dataset}
We evaluate the proposed Bayesian CycleGAN on three widely used benchmarks, Cityscapes \cite{cordts2016cityscapes}, Maps and Monet2Photo \cite{CycleGAN2017}. We use Cityscapes for learning mapping between realistic photos and semantic labels, with 2975 training images from training set and 500 testing images from validation set. Each picture is resized to $128\times256$. We use Maps for mapping learning between aerial photographs and maps scraped from Google Maps, with 1096 training images and 1098 testing images. Each picture is resized to $256\times256$. We use Monet2Photo for mapping learning between landscapes and artworks (style transfer), with 1074 Monet artworks and 6853 landscapes.

\subsection{Implementation Details}
We employ the same network architectures for generators and discriminators and same hyper-parameters for training process.
The generator consists of 6 basic residual blocks \cite{he2016deep} between two down-sample layers and two up-sample layers, assembled by instance normalization. The discriminator consists of several convolutional layers.
In the training process, we use ADAM \cite{kingma2014adam} with learning rate 0.0002 for the first 50 epochs, the learning rate will nonlinearly decay to zero periodically for the next 50 epochs, momentums were set as $\beta_1=0.5$ and $\beta_2=0.999$. Weights for networks are initialized as Gaussian distribution with zero mean and standard deviation 0.02. The other hyper-parameters, default values are $\lambda=10$ and $\lambda_{KL}=0.1$. Besides, we set Monte Carlo latent sample numbers as 3 for both $f_x$ and $f_y$ for fixed volume variant concatenation.

\begin{figure*}
\begin{minipage}[t]{0.5\linewidth}
    \begin{flushleft}
    \includegraphics[width=0.98\linewidth,height=7cm]{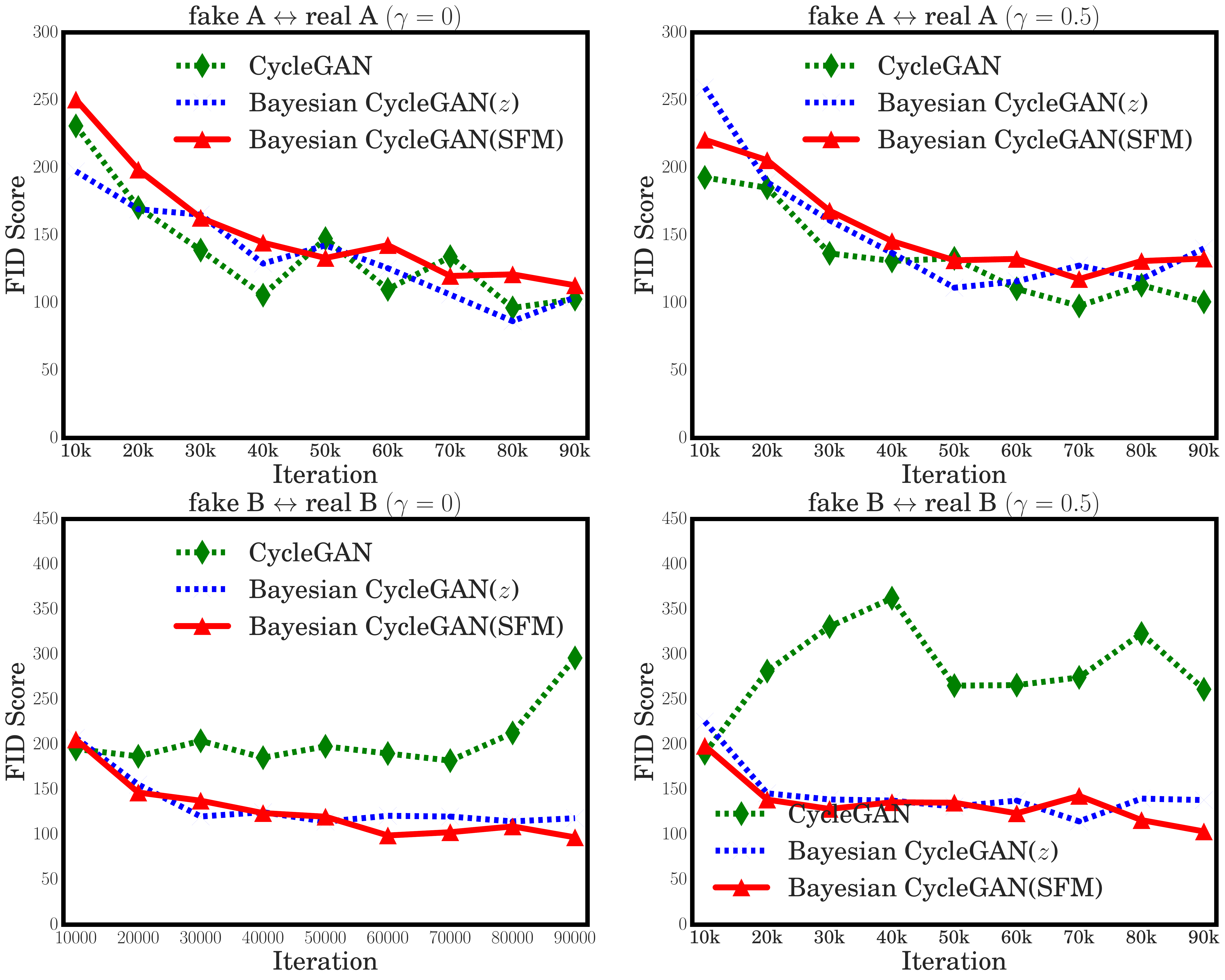}
    \end{flushleft}
\end{minipage}
\begin{minipage}[t]{0.5\linewidth}
    \begin{flushright}
    \includegraphics[width=0.98\linewidth,height=7cm]{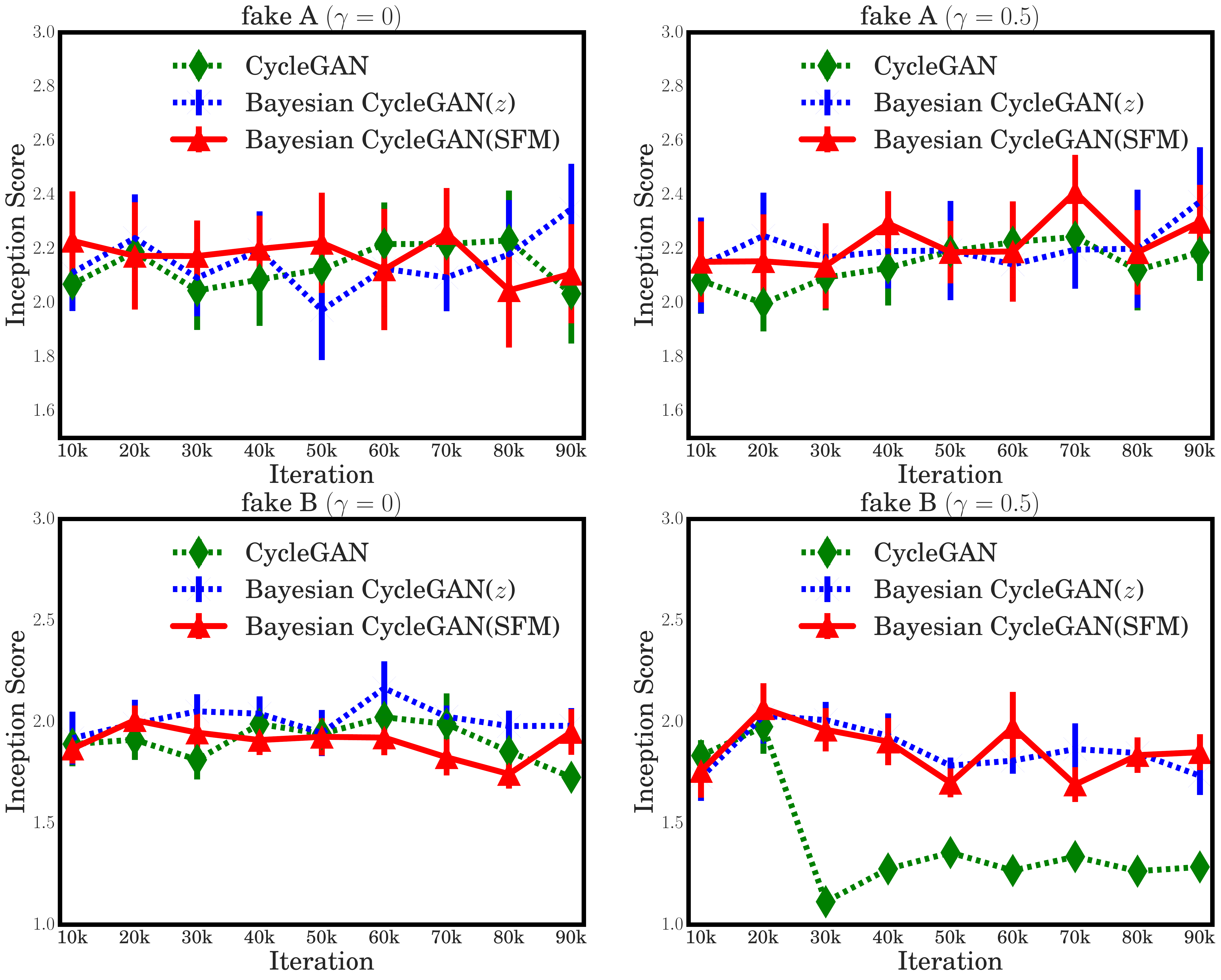}
    \end{flushright}
\end{minipage}
\caption{
Left two columns show FID Scores to estimate the \emph{statistical distance} between source images and generated ones during different iterations. Right two columns show Inception Scores of different iterations with mean value and standard deviation. In all above figures: $A$ represents the streetscapes; $B$ signifies semantic segmentation labels.
}
\vspace{-1em}
\label{fid-inception}
\end{figure*}

\begin{figure}[!htbp]
    \centering
    \includegraphics[width=\linewidth]{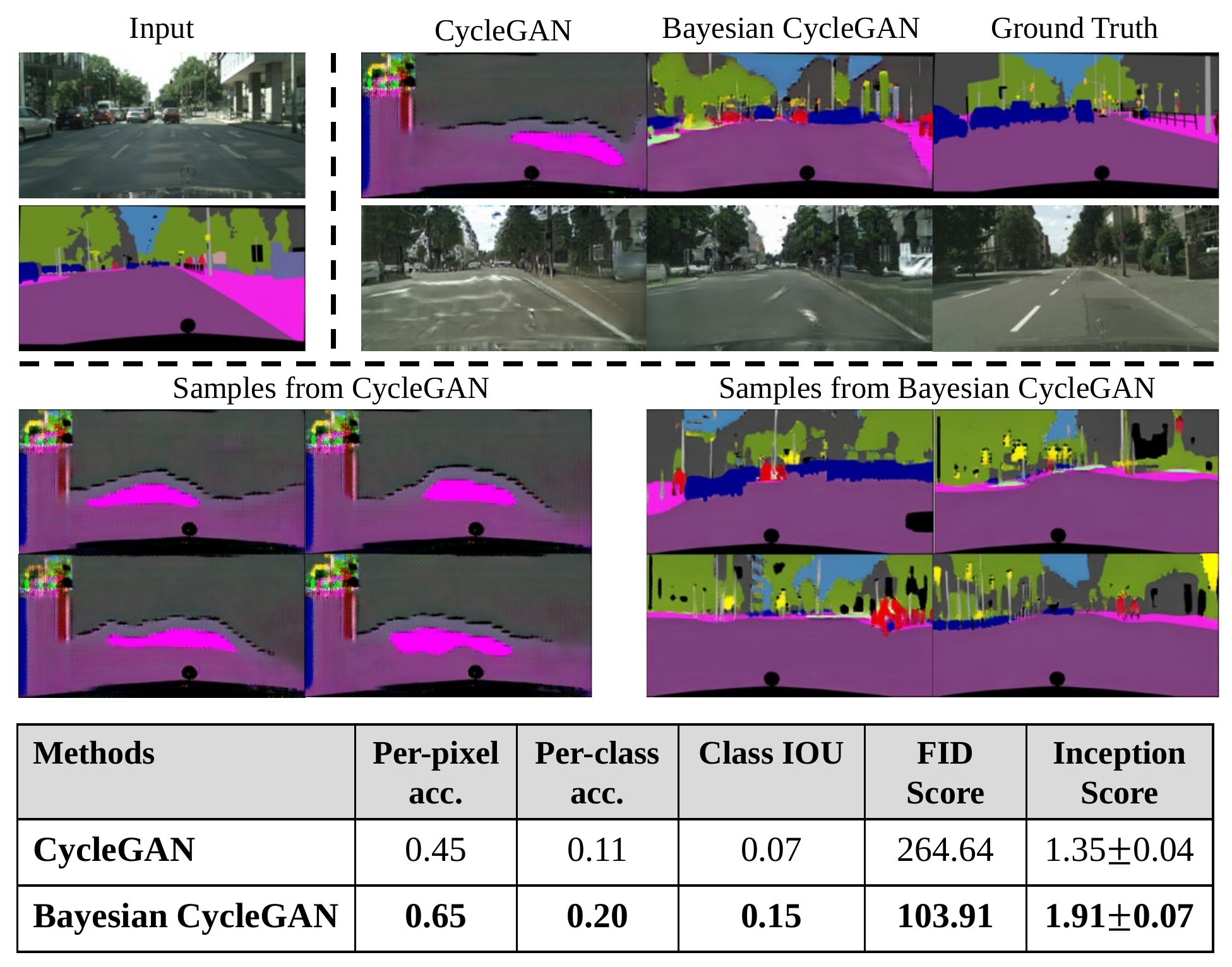}
\caption{
Comparison of stability: Results for mapping $label \leftrightarrow photo$ trained on cityscapes under condition $\gamma=0.5$. The upper panels show the translation results while the bottom panels show generated label samples from different methods.
}
\label{fig:cityscapes}
\end{figure}

\begin{figure}[t]
\centering
\includegraphics[width=\linewidth]{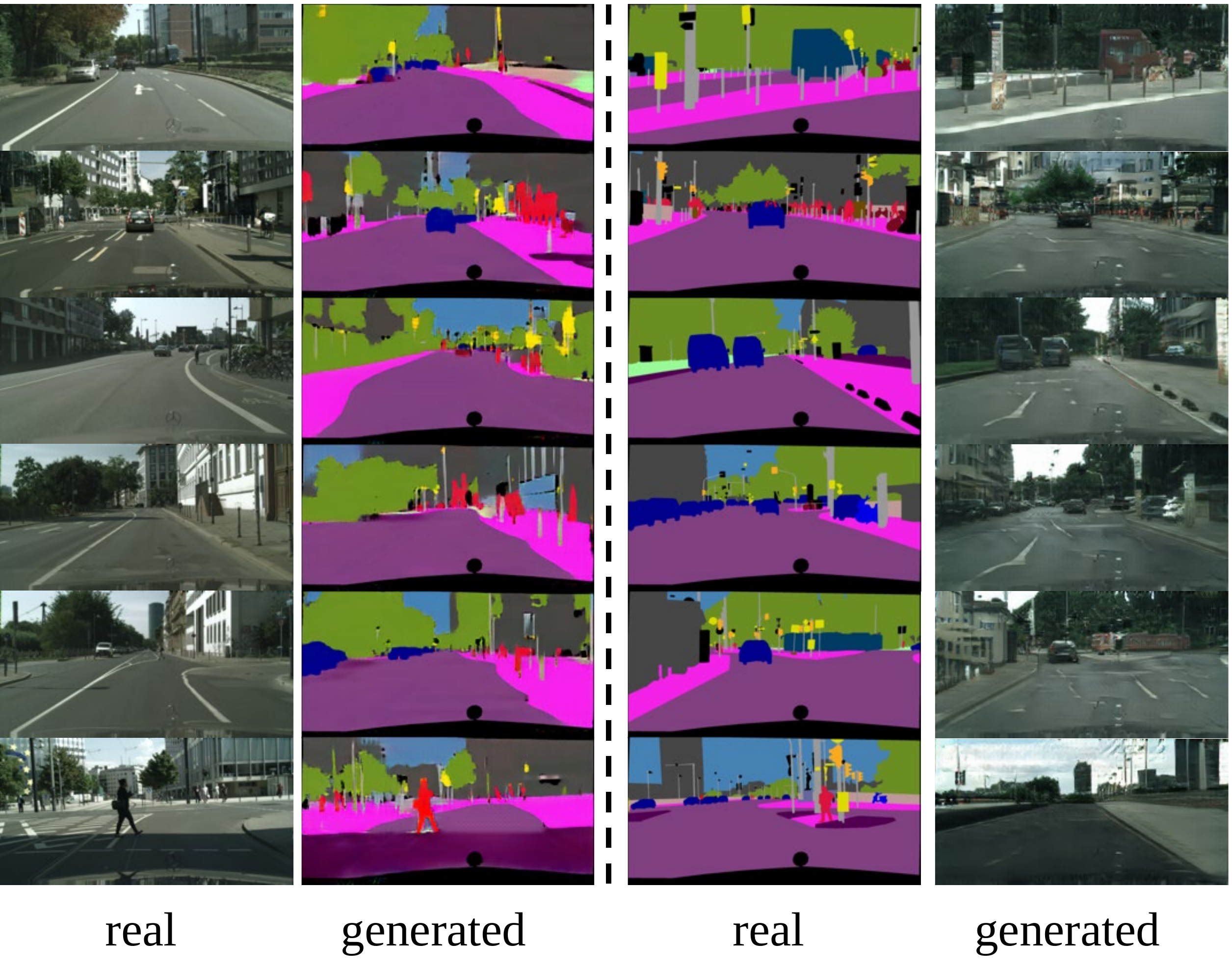}
\caption{Qualitative results sampling form Bayesian cyclic model in unsupervised setting under condition $\gamma=0$.}
\label{fig:cityscapes_2}
\end{figure}

\subsection{Comparison Based on Standard Metrics}
\label{sec:compare}

\noindent
\textbf{Evaluation Metrics} We evaluate the quality of generated images and statistical distance between generated images and source images using the metrics below:
\begin{itemize}
    \item \underline{\textit{Inception Score (IS)}}: the metric for measuring the diversity and quality of generated samples. Salimans et al. \cite{salimans2016improved} applies the pretrained Inception model \cite{szegedy2016rethinking} to generate samples and then compares conditional label distribution with the marginal label distribution: $\text{exp}(\mathbb{E}_{x \sim p_{g}} D_{KL}(p(y|x) \| p(y)))$. Generating samples with meaningful objects reduces the entropy of conditional label distribution while generating diverse images increase the entropy of marginal label distribution. Thus higher scores are better, corresponding to a larger KL-divergence between the two distributions (meaningful and diverse).
    \item \underline{\textit{FID Score}}: Heusel et al. \cite{heusel2017gans} propose a complementary metric of IS called ``Fre{\'c}het Inception Distance (FID)'' for better capturing the similarity of generated images to real ones. Thus lower FID score is better, corresponding to more similar real and generated samples.
    \item \underline{\textit{FCN Scores}}: We use the standard metrics from the Cityscapes benchmark \cite{cordts2016cityscapes} to evaluate the performance of photo $\rightarrow$ label translation following CycleGAN \cite{CycleGAN2017}, including per-pixel accuracy, per-class accuracy, and mean class Intersection-Over-Union (IOU).
\end{itemize}

\noindent\textbf{Photo $\leftrightarrow$ Label Translation}
We adopt Cityscapes dataset representing semantic-level domain translation to evaluate the stability and performance of proposed Bayesian CycleGAN,
the label domain denotes the semantic segmentation.
We first evaluate different methods using FID Score and IS in both original cyclic framework $(\gamma=0)$ and proposed integrated cyclic framework $(\gamma > 0)$, the comparison are shown in Fig. \ref{fid-inception} (similar analysis applied to semi-supervised setting Fig. \ref{fig:semi_fid_inception}),
from which we find that various methods achieve similar performance
(low FID Score and high IS)
for easy task $label \rightarrow photo$
(upper panels).
However, for semantic segmentation task $photo \rightarrow label$ (bottom panels), Bayesian CycleGAN outperforms (lower FID Score and higher IS) original CycleGAN in both generated diversity/quality and training stability.
The difference between the up (label $\rightarrow$ photo) and bottom (photo $\rightarrow$ label) subfigures in Fig. 3 attributes to that \textbf{GANs are not equally effective for different translation tasks}. As Xue et al.~\cite{xue2018segan} indicates, ``When inputs to the discriminator are generated vs. ground truth dense pixel-wise label maps as in the segmentation task, the real/fake classification task is too easy for the discriminator and a trivial solution is found quickly. As a result, no sufficient gradients can flow through the discriminator to improve the training of generator.'' So the ``photo $\rightarrow$ label'' task provides a harder scenario for GANs to balance generator and discriminator, which is the reason why Bayesian CycleGAN performs better in the bottom task.
As mentioned before, balance factor is introduced to adjust the stability of cyclic framework by accelerating the learning process of discriminator.
Therefore, to gain more insights of the improved training stability, we conduct ablation studies of introduced balance factor and visualize the samples generating from different methods in Fig. \ref{fig:cityscapes}. We can see that CycleGAN already suffers from mode collapse under $\gamma = 0.5$ scenario, generating identical and meaningless label maps given any input streetscape; while Bayesian CycleGAN still works well. The comparison between them demonstrates that proposed Bayesian CycleGAN can enhance generator training by employing posterior sampling (use more diverse inputs to boost robustness) and regularized priors (add regularizaiton to avoid over-fitting), resistant to crazy learning discriminator.

\begin{figure*}[t]
    \begin{minipage}[t]{0.7\linewidth}
        \includegraphics[width=\linewidth, height=3.5cm]{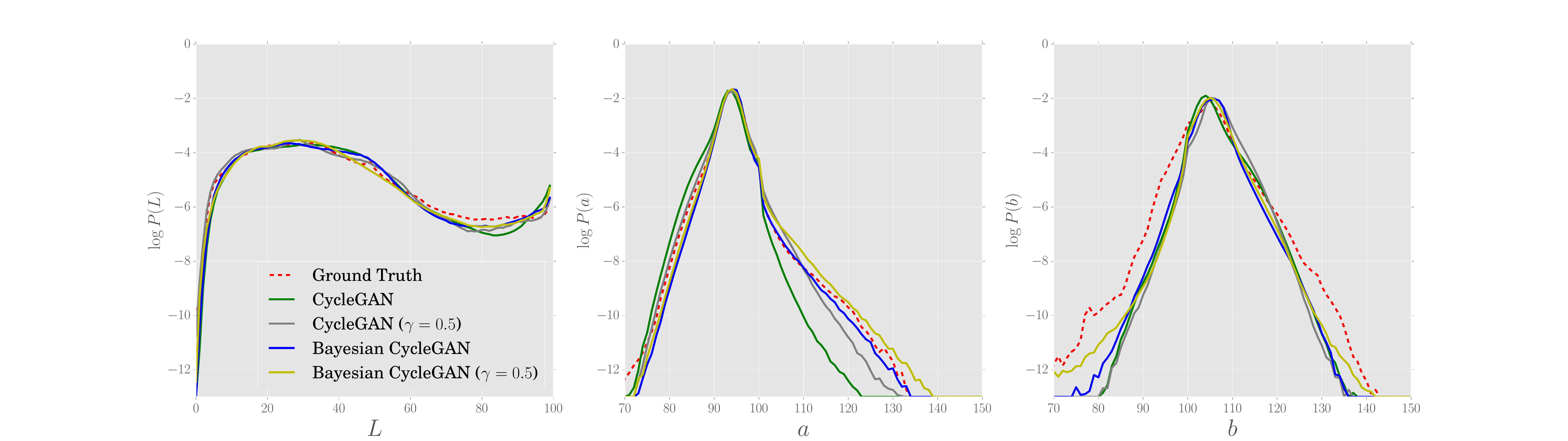}
    \end{minipage}
    \begin{minipage}[t]{0.3\linewidth}
        \vspace{-3cm}
        \centerline{Histogram intersection}
        \centerline{against ground truth}
        \resizebox{\linewidth}{10mm}{
            \begin{tabular}{|c|l|c|c|c|c|}
                \hline
                $\gamma$ & \textbf{Methods} & \textbf{L} & \textbf{a} & \textbf{b} \\ \hline\hline
                \multirow{2}*{0}
                 & CycleGAN & 0.97 & 0.81 & 0.78 \\
                 & Bayesian CycleGAN & {\bf 0.99} & {\bf 0.98} & {\bf 0.80} \\
                \hline
                \multirow{2}*{0.5}
                 & CycleGAN & 0.98 & 0.92 & 0.75 \\
                 & Bayesian CycleGAN & {\bf 0.98} & {\bf 0.97} & {\bf 0.84} \\
                \hline
            \end{tabular} }
    \end{minipage}
    \caption{
    Left: Color distribution in Lab color space matching property of different methods, tested on Cityscapes. We show the log probability for emphasizing the differences in low probability regions. Right: Histogram intersection against ground truth.}
    \label{fig:histogram}
\end{figure*}

We also adopt automatic quantitative FCN Scores (Cityscapes Segmentation Metrics) for evaluating $photo \rightarrow label$ task. The results are shown in Table \ref{cityscapes}. Proposed Bayesian CycleGAN outperforms the CycleGAN in different degrees of difficulty for stabilizing training
(i.e., different balance factor). It shows a significant improvement over semantic segmentation task without supervised information, achieving \textbf{15\%} gain on per-pixel accuracy and \textbf{4\%} gain on class IOU as compared to CycleGAN.
Also, when it comes to integrated cyclic framework where balance factor makes training more difficult to be stable, Bayesian model shows greater resilience to model collapse and improves \textbf{20\%} for per-pixel accuracy and \textbf{8\%} for Class IOU. This performance demonstrates the superior stability of proposed method.
In addition to the quantitative results, we provide some sampled qualitative results shown in Fig. \ref{fig:cityscapes_2}.

\begin{figure}[!htbp]
\begin{center}
    \includegraphics[width=\linewidth]{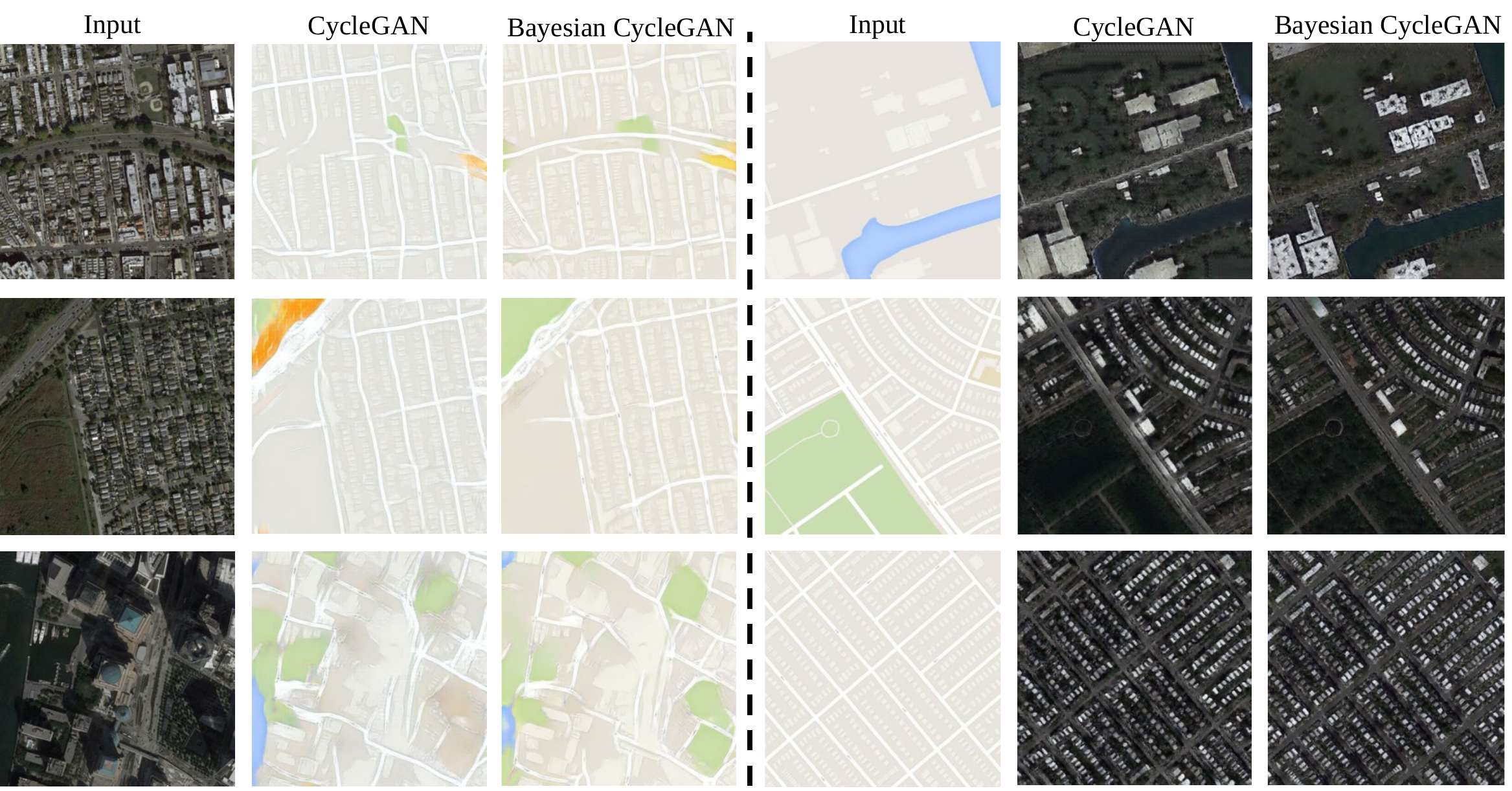}
\end{center}
\caption{Maps $\leftrightarrow$ Aerial: qualitative results conditioned on $\!\gamma\!=\!0$.}
\label{fig:maps}
\end{figure}

\begin{table}[!htbp]
\begin{center}
\resizebox{\linewidth}{15mm} {
\begin{tabular}{|c|l|c|c|c|c|}
\hline
$\gamma$ & \textbf{Methods} & \textbf{Per-pixel acc.} & \textbf{Per-class acc.} & \textbf{Class IOU} \\ \hline\hline
\multirow{4}*{0}
 & CycleGAN & 0.48 & 0.18 & 0.11 \\
 & CycleGAN (dropout) & 0.56 & 0.18 & 0.12 \\
 & CycleGAN (buffer) & 0.58 & 0.22 & 0.16 \\
 & Bayesian CycleGAN & {\bf 0.73} & {\bf 0.27} & {\bf 0.20} \\
\hline
\multirow{3}*{0.5}
 & CycleGAN & 0.45 & 0.11 & 0.07 \\
 & CycleGAN (dropout) & 0.59 & 0.16 & 0.11 \\
 & Bayesian CycleGAN & {\bf 0.65} & {\bf 0.20} & {\bf 0.15} \\
\hline
& Pix2pix (Supervised) & 0.85 & 0.40 & 0.32 \\
\hline
\end{tabular} }
\end{center}
\caption {FCN Scores for different methods, evaluated on photo$\rightarrow$label. \emph{Buffer} is the trick used by Zhu et al. \cite{CycleGAN2017} For fair comparison, we evaluate all methods after 200 epochs.}
\label{cityscapes}
\end{table}

\noindent\textbf{Aerial $\leftrightarrow$ Maps Translation} We evaluate different methods using FID and inception score in Table \ref{tab:maps},
from which we find that Bayesian cyclic model achieves nearly 5\% drop in FID Score metric (lower is better) and 5\% improvement in IS metric (higher is better).
In addition, the qualitative results in Fig. \ref{fig:maps} demonstrate the realistic effect of Bayesian CycleGAN.

\begin{table}[!htbp]
\begin{center}
\resizebox{\linewidth}{9mm} {
\begin{tabular}{|c|l|c|c|c|c|c|}
\hline
$\gamma$ & \textbf{Methods} & \textbf{FID A} & \textbf{FID B} & \textbf{Inception A} & \textbf{Inception B} \\ \hline\hline
\multirow{2}*{0}
 & CycleGAN & 71.56  & 172.75  & 3.50$\pm$0.25 & 2.77$\pm$0.14 \\
 & Bayesian CycleGAN & {\bf 68.45}  & \textbf{165.83} & \textbf{3.68$\pm$0.22} & \textbf{2.85$\pm$0.15} \\
\hline
\multirow{2}*{0.5}
 & CycleGAN & 77.92  & 198.26  & 3.43$\pm$0.11 & 2.56$\pm$0.16 \\
 & Bayesian CycleGAN & \textbf{72.86}  & \textbf{160.66} & \textbf{3.61$\pm$0.32} & \textbf{2.83$\pm$0.12} \\
\hline
\end{tabular} }
\end{center}
\caption {Evaluation for Maps$\leftrightarrow$Aerial mapping. FID A evaluates statistic distance between generated aerial photos and real photos, FID B is between generated maps and real maps. Inception score measures the quality of generated images.}
\label{tab:maps}
\end{table}

\begin{figure*}[t]
\begin{center}
    \includegraphics[width=\linewidth]{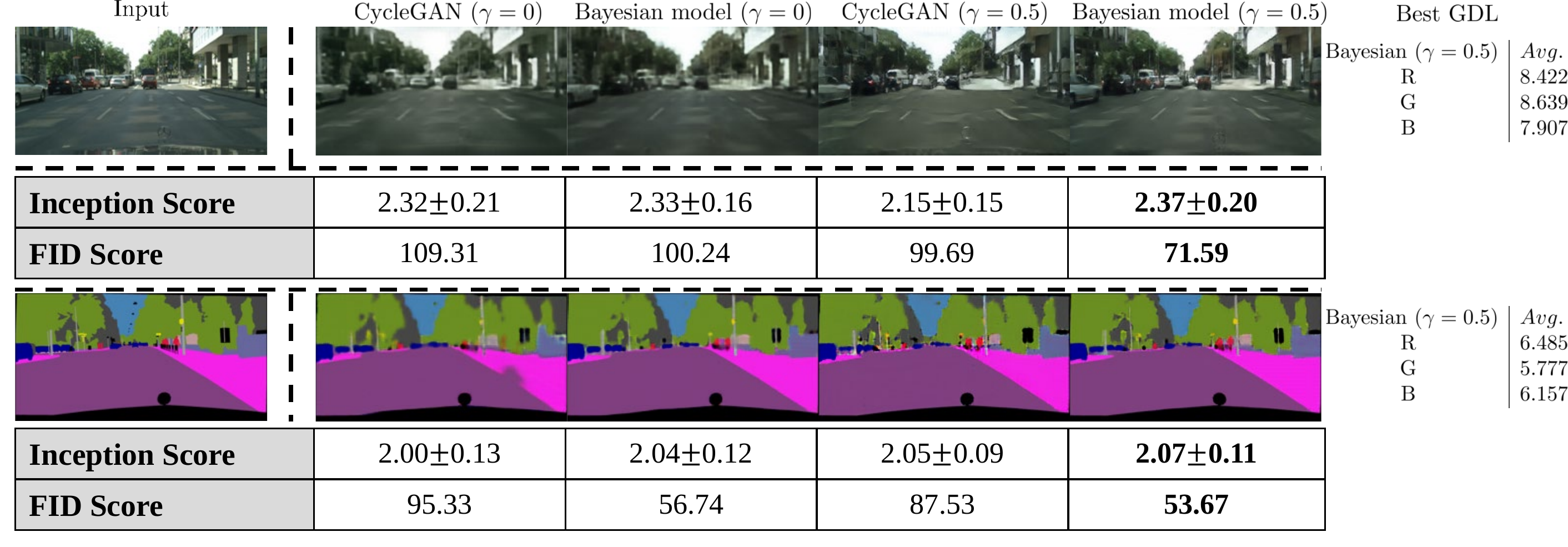}
\end{center}
\vspace{-1em}
\caption{
We use different methods under different conditions for $label \leftrightarrow photo$ mapping. Then we infer reconstructed images after training 50000 iterations and calculate average GDL, IS and FID Score between reconstructed images and source images.
\vspace{-2em}
}
\label{cityscapes_rec}
\end{figure*}

\noindent\textbf{Monet $\leftrightarrow$ Photo Translation} is a kind of image style transfer. Instead of focusing on the neural transfer \cite{Gatys2016Image} that mimic the style of only one painting, we employ automatic mapping and compare different methods in Table \ref{tab:monet2photo}. With the similar analysis as above, Bayesian CycleGAN shows superior performance. Its diversified generating effects are illustrated in Section \ref{diversified}.

\begin{table}[!htbp]
\begin{center}
\resizebox{\linewidth}{6.5mm} {
\begin{tabular}{|l|c|c|c|c|}
\hline
\textbf{Methods} & \textbf{FID A} & \textbf{FID B} & \textbf{Inception A} & \textbf{Inception B} \\ \hline\hline
CycleGAN & 155.32 & 140.00  & 3.76$\pm$0.73 & 3.17$\pm$0.35 \\
Bayesian CycleGAN & \textbf{151.24} & \textbf{137.74} & \textbf{3.86$\pm$0.64} & \textbf{3.35$\pm$0.34} \\
\hline
\end{tabular} }
\end{center}
\caption{Evaluation for Monet $\leftrightarrow$ Photo mapping under condition $\gamma=0$. FID A evaluates statistic distance between generated images in Monet style and artworks; FID B is between generated photos and landscapes.}
\label{tab:monet2photo}
\end{table}
\vspace{-0.3in}

\subsection{Boost Realistic Color Distribution}
\begin{figure}[h]
    \centering
    \includegraphics[width=\linewidth]{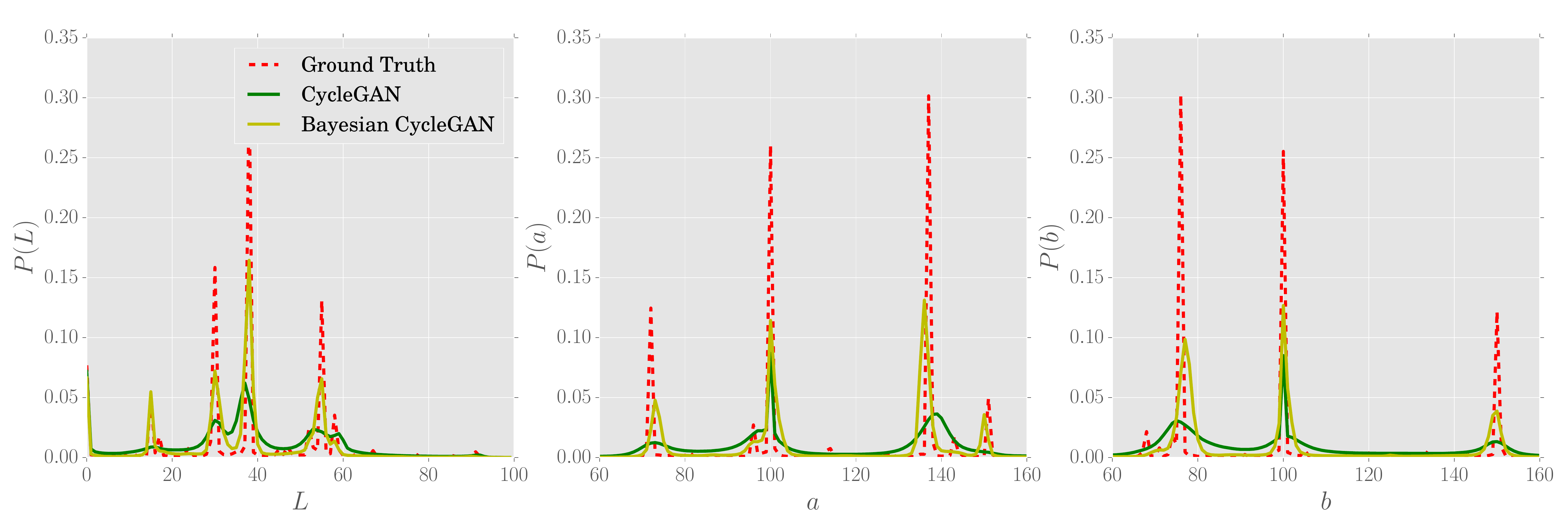}
    \vspace{-1em}
    \caption{
    Color distribution matching property of different methods, tested on Cityscapes.}
    \vspace{-2em}
    \label{semantic_color_dist}
\end{figure}
Another striking effect of Bayesian CycleGAN is of generating realistic landscapes with extensive color distribution.
As demonstrated in \cite{pix2pix2017}, L1 loss stimulates an average, grayish color when it comes to uncertainty space, while adversarial loss encourages sharp and realistic color distribution.
The sampling and regularized priors in our method enhance the sharpening effect of adversarial loss in gradient descend process, which in turn make generated image more colorful.
In addition, the introduced adversarial loss in reconstructed ends also incentivize matching the true color distribution.
In Fig. \ref{fig:histogram}, we investigate if our Bayesian CycleGAN actually achieve this effect on the Cityscapes dataset. The plots show the marginal distributions over generated color values in Lab color space.
We find that: 1) Bayesian CycleGAN generates almost unanimous La distribution with ground truth, and encourages broader b distribution than original methods; 2) The introduced balance factor leads to a wider distribution, confirming the hypothesis that the Bayesian formulation and introduced balance factor both contribute to generating closer color distribution.

We have shown Bayesian CycleGAN generates realistic landscapes with extensive color distribution. Now we present the comparison of semantic color distribution in Fig. \ref{semantic_color_dist}.
Since the sampling and regularized priors in Bayesian CycleGAN enhance the sharpening effect of adversarial loss in gradient descend process, which in turn make generated semantic distribution closer to the true data distribution (See Fig. \ref{fig:cityscapes}),
Bayesian CycleGAN can then boost sharpening and realistic marginal distribution over generated semantic color values in Lab color space, shown in Fig. \ref{semantic_color_dist}.
\vspace{-0.1in}

\vspace{-0.1em}
\subsection{Enhance Reconstructed Learning}
We have shown the effect of balance factor $\gamma$ for adjusting the stability of cyclic framework. Since the GAN loss is added between real images and reconstructed images, which are more realistic than generated ones, the return gradients for generators are relatively smaller than that for discriminators so that the min-max adversarial game goes out of balance.
That explains why original cyclic model suffers from mode collapse as illustrated in Fig. \ref{fig:cityscapes}, and it also demonstrates Bayesian CycleGAN is more stable.
Now, we present the other function of $\gamma$: enhancing reconstructed learning.

In theory, by introducing the balance factor, the optimal status of cyclic model changes from $p_y(y) = p_{\tilde{y}}(y)$ to $(1+\gamma)p_y(y) = p_{\tilde{y}}(y) + \gamma p_{\hat{y}}(y)$, which encourages the distribution of reconstructed images closer to real data distribution. In practice, we compare the FID scores of reconstructed distribution under different settings, illustrated in Fig. \ref{fid_rec}. The FID scores for cyclic model conditioned on $\gamma=0.5$ are lower compared with that conditioned on $\gamma=0$, which means the balance factor enhances reconstructed learning.

In addition, we exhibit reconstructed images from the perspective of both qualitative and quantitative analysis and evaluate images with GDL, Gradient Difference Loss \cite{fabbri2018enhancing}, IS and FID Score, shown in Fig. \ref{cityscapes_rec}, in which lower GDL means that the reconstructed images create a sense of verisimilitude and have similar semantic expression with real images.
The used three metrics demonstrate Bayesian CycleGAN conditioned on $\gamma=0.5$ achieves the best performance.
In particular, Bayesian CycleGAN with adversarial loss in the reconstructed ends ($\gamma=0.5$) achieves \textbf{up to} $\textbf{28.6\%}$ improvement as compared to its counterpart without that loss ($\gamma=0$) based on the FID Score measured between reconstructed images and real images, \textbf{without harming the quality of generated samples thanks to our Bayesian framework} (i.e., remains comparable performance with its counterpart in terms of both FID Score and Inception Score measured between generated images and real images as shown in Fig. 3).
\begin{figure}[!htbp]
\begin{center}
\includegraphics[width=\linewidth,height=3.2cm]{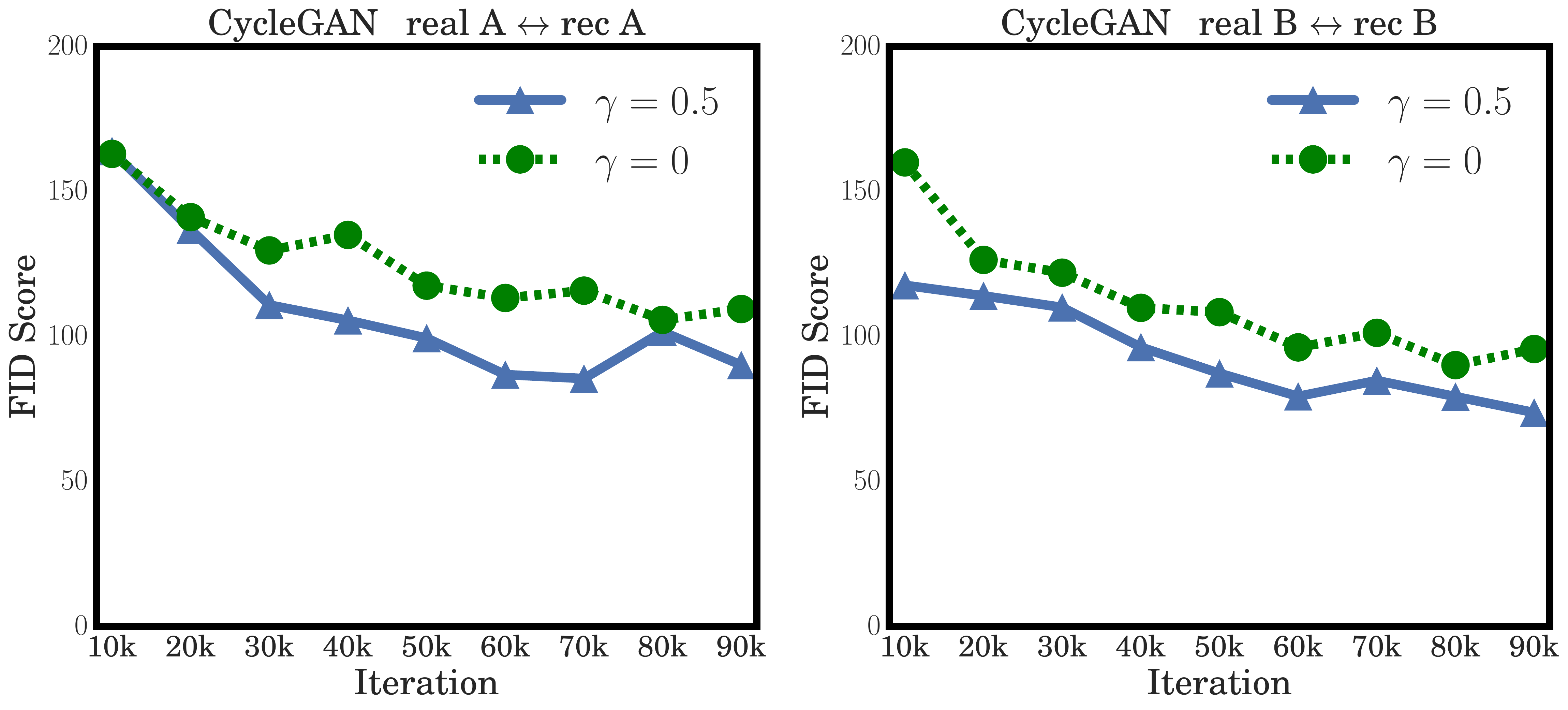}
\includegraphics[width=\linewidth,height=3.2cm]{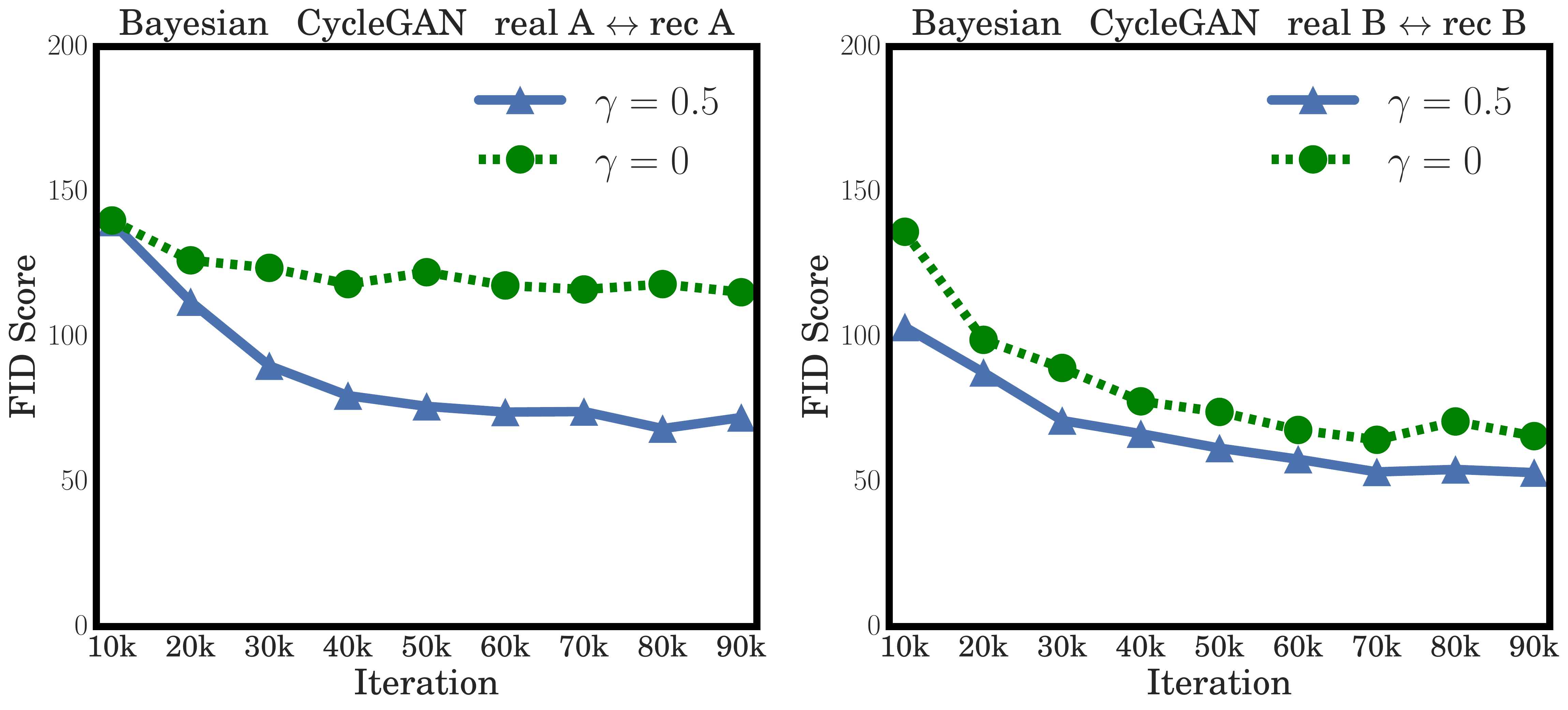}
\end{center}
\vspace{-1em}
\caption{
The FID scores estimate \emph{distance} between the distribution of source images and that of reconstructed images generated by original CycleGAN / Bayesian CycleGAN.
}
\vspace{-1em}
\label{fid_rec}
\end{figure}

\begin{figure*}[t]
\begin{minipage}[t]{0.5\linewidth}
    \begin{flushleft}
    \includegraphics[width=0.98\linewidth]{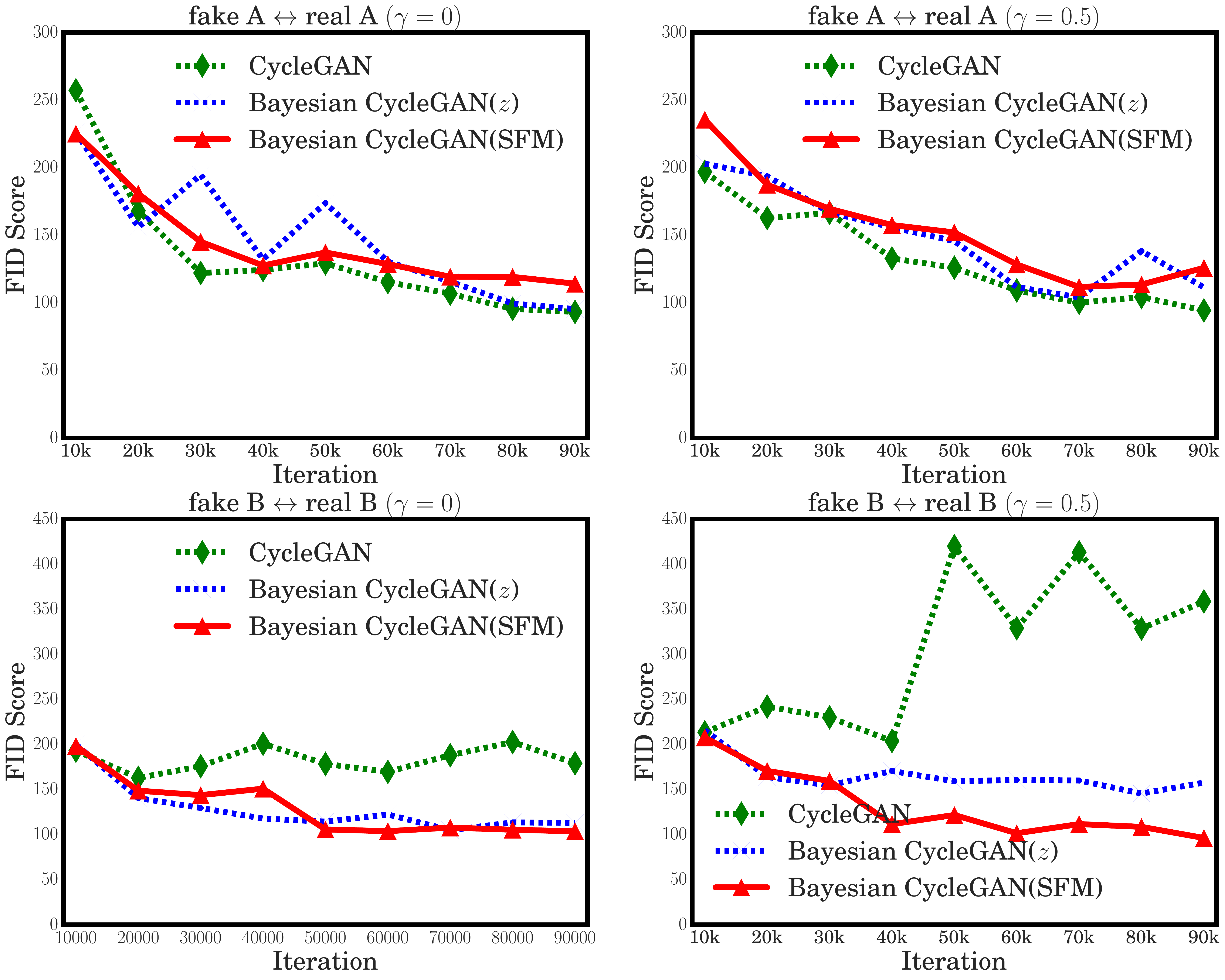}
    \end{flushleft}
\end{minipage}
\begin{minipage}[t]{0.5\linewidth}
    \begin{flushright}
    \includegraphics[width=0.98\linewidth]{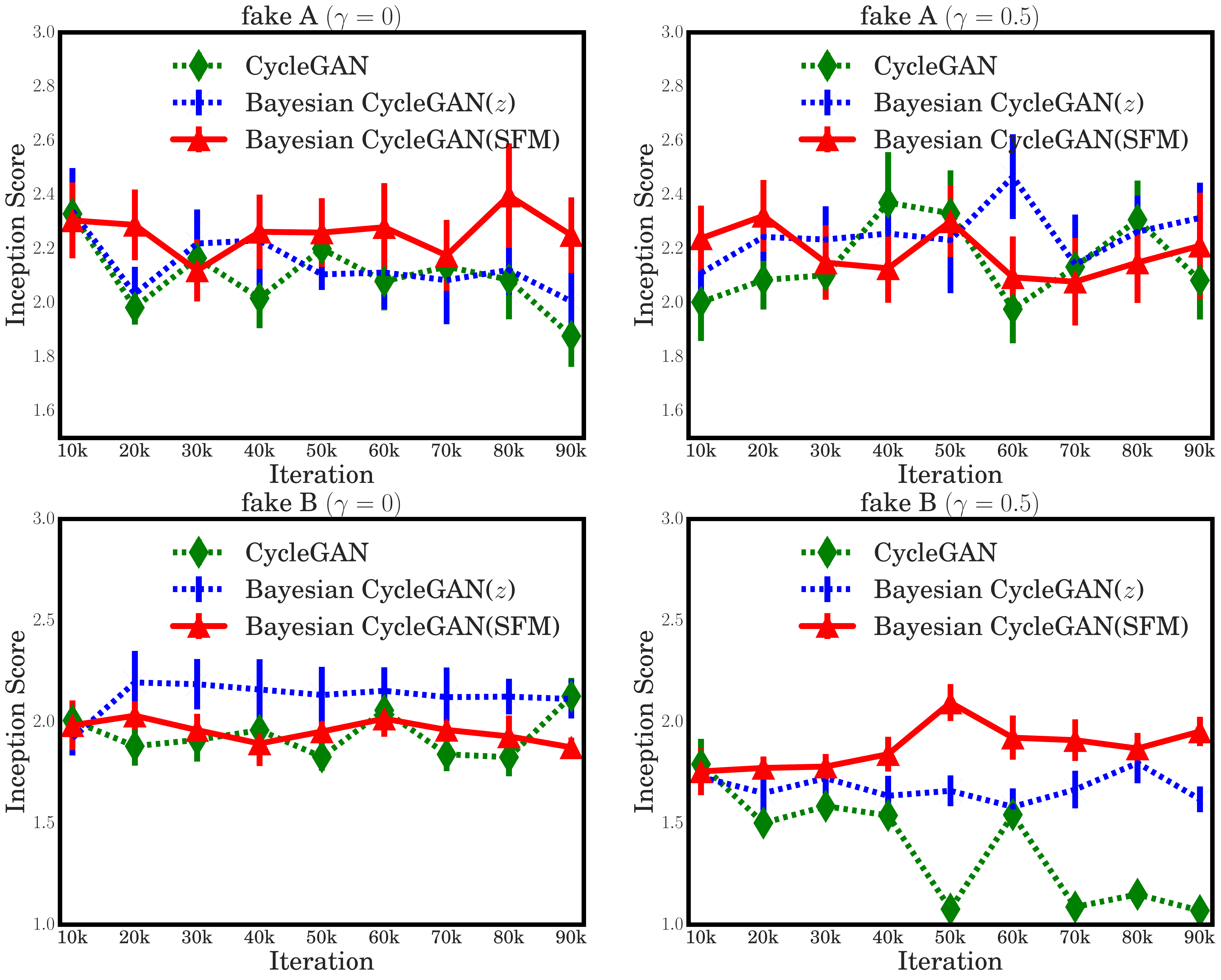}
    \end{flushright}
\end{minipage}
\caption{
Quantitative results for semi-supervised learning (use $1\%$ paired data). Left: The FID Scores estimate the \emph{distance} between the distribution of real images and that of generated ones. Right: The Inception Score for evaluating the quality of generated images. The analysis and conclusion in semi-supervised setting are consistent/same with Section \ref{sec:compare}. }
\label{fig:semi_fid_inception}
\end{figure*}

\vspace{-1em}
\subsection{Diversify Outputs by Replacing Restricted Variable}
\label{diversified}
As mentioned in \cite{almahairi2018augmented}, stochastic CycleGAN cannot generate diversified outputs since the cycle-consistency loss encourages the mapping to ignore latent variables.
Think out of the box, for generating diversified outputs, we can divide latent space into various parts and use only one particular part to fit in training process. Therefore, in the inference process, it is possible for us to diversify generating outputs by replacing that part with another latent distribution.
In practice,
we impose restriction on sampled latent space by changing the latent variable from random noise $z$ to the statistical feature maps generated by VAE-like network illustrated in Section \ref{restricted_latent_sampling}. In inference process, by replacing the statistical feature map with other latent variables, cyclic model can then generate diversified outputs. We carry out several tests on $monet \leftrightarrow photo$ style transfer experiment. The diversified outputs and corresponding latent variables are shown in Fig. \ref{fig:monet2photo}.

\begin{figure}[t]
\begin{center}
    \includegraphics[width=\linewidth]{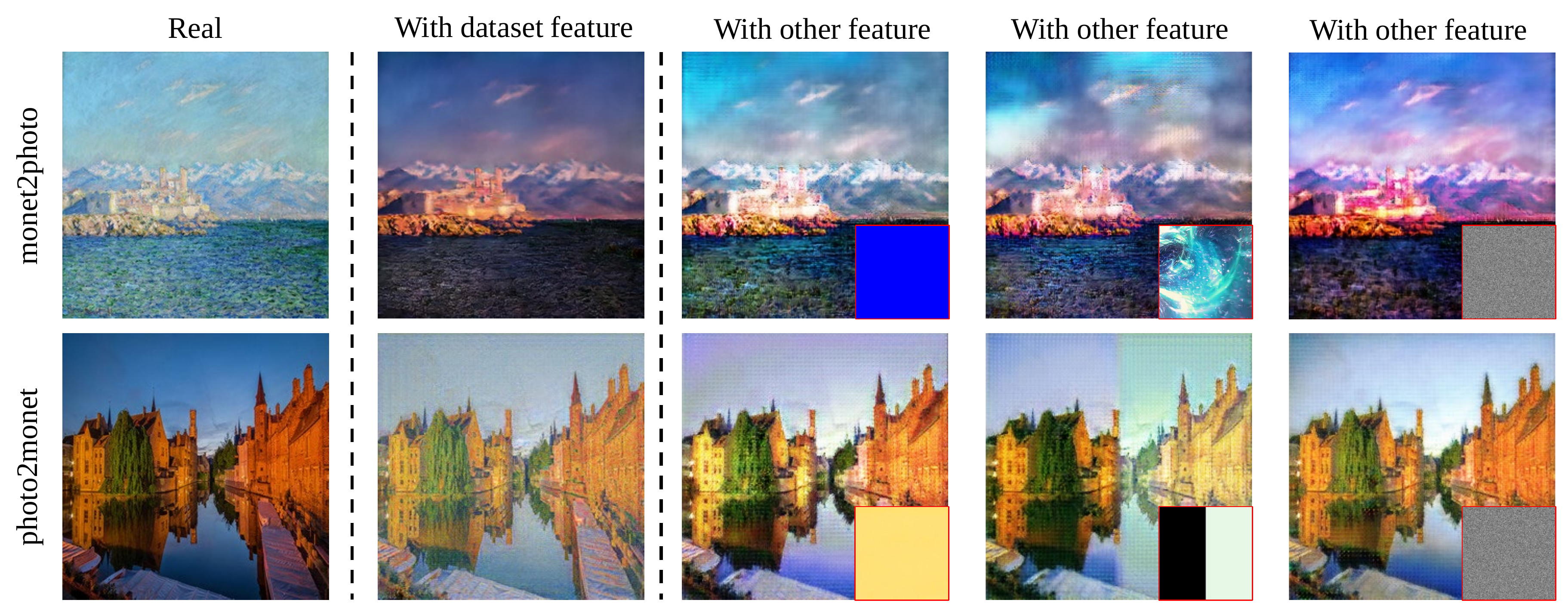}
\end{center}
\caption{
Monet $\leftrightarrow$ Photo Translation: The left column are source images, middle column are generated ones, right columns are diversified outputs generated by replacing SFM with other latent variables, shown on the lower right corner of pictures.}
\vspace{-1em}
\label{fig:monet2photo}
\end{figure}

\subsection{Adaptability in Semi-Supervised Setting}
\label{semi-supervised}
In cases where paired data is accessible, we can leverage the condition to train our model in a semi-supervised setting. By dong this, we could get more accurate mapping in several complicated scenario. In the training process of Cityscapes, mapping errors often occur, for example, the initial model cannot recognize trees, thus, translating trees into something else due to the unsupervised setting. To resolve these ambiguities requires weak semantic supervision, we can use 30 (around $1\%$) paired data (pictures of cityscape and corresponding label images) to initialize our model at the beginning for each epoch. The results (Fig. \ref{fig:semi_fid_inception}) are similar to unsupervised learning, but we find that this slight priori knowledge increases the accuracy of mapping between photos and labels significantly.

\vspace{-1em}
\section{Conclusion}
\label{conclusion}
In this paper,
targeting two underlying assumptions in current cylcic frameworks,
we first propose Bayesian CycleGAN and an integrated cyclic framework for inter-domain mapping, aiming to alleviate the stability issue in cyclic framework.
By exploring the whole posteriors with latent sampling and regularized priors, Bayesian CycleGAN can enhance the generator training and therefore alleviate the risk of mode collapse and instability caused by the unbalance of min-max optimization game.
This is the considerable reason why we achieved more stable training and better domain-to-domain translation results.
Meanwhile, the proposed integrated cyclic framework enable realistic reconstructed learning in theory and practice by replacing the $l_1$ loss with a combination of $l_1$-GAN loss.
In addition, we propose retricted latent sampling method to break the curse of deterministic mapping in CycleGAN.
After imposing restriction on the sampled latent variables, we can achieve diversified mappings by replacing the SFM with other latent variables in inference process.

The proposed Bayesian CycleGAN significantly improves performance and stability.
The next orientation is how to achieve controllable many-to-many mapping under this framework. Also, modifying prior models according to dynamic posterior estimation is worthy of future research.

\bibliographystyle{unsrt}
\bibliography{egbib}

\newpage

\begin{IEEEbiography}
[{\includegraphics[width=1in, height=1.25in]{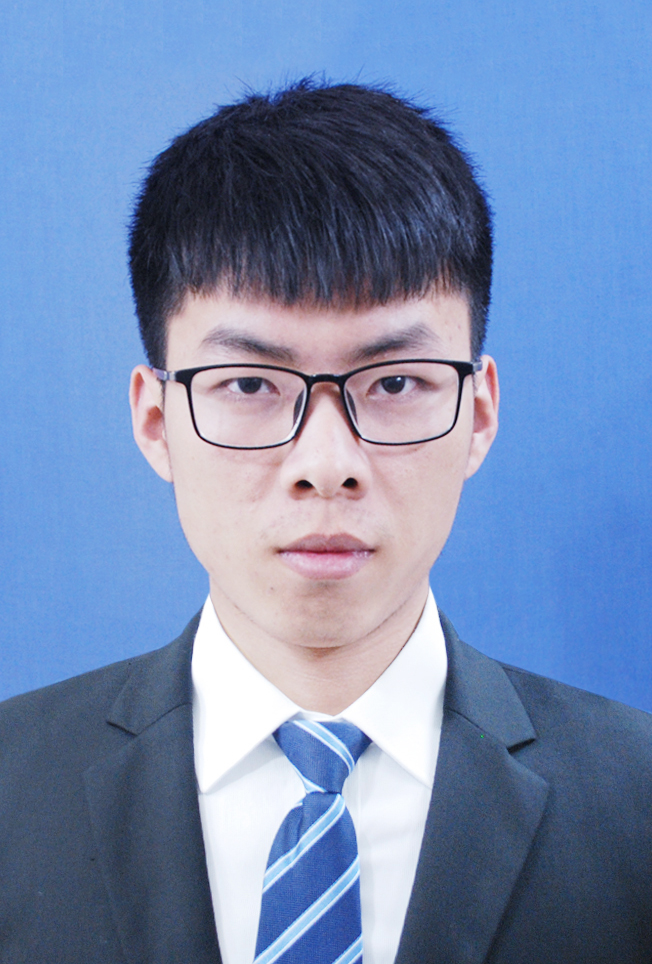}}]{Haoran You}
is currently a Ph.D. student in Electronic and Computer Engineering Department of Rice University. He received his bachelor degree at School of Electronic Information and Communications, Huazhong University of Science and Technology, Wuhan, P.R. China. During undergraduate time, he mainly focus on the area of generative model. Now he is pursuing his doctoral degree in machine learning realm. His research interests include Computer vision, deep learning and resource-constrained machine learning.
\end{IEEEbiography}

\begin{IEEEbiography}
[{\includegraphics[width=1in, height=1.25in]{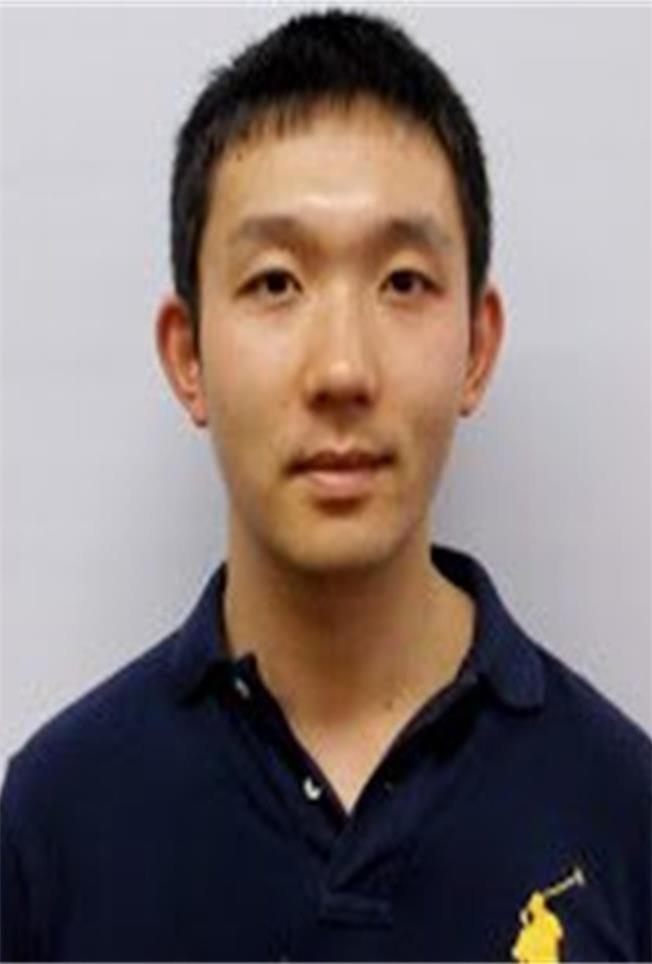}}]{Yu Cheng}
is a Researcher at Microsoft. Before that, he was a Research Staff Member at IBM T.J. Watson Research Center. Yu got his Ph.D. from Northwestern University in 2015 and bachelor from Tsinghua University in 2010. His research is about deep learning in general, with specific interests in the deep generative model, model compression, and adversarial learning. He regularly serves on the program committees of top-tier AI conferences such as NIPS, ICML, ICLR, CVPR and ACL.
\end{IEEEbiography}

\begin{IEEEbiography}
[{\includegraphics[width=1in, height=1.25in]{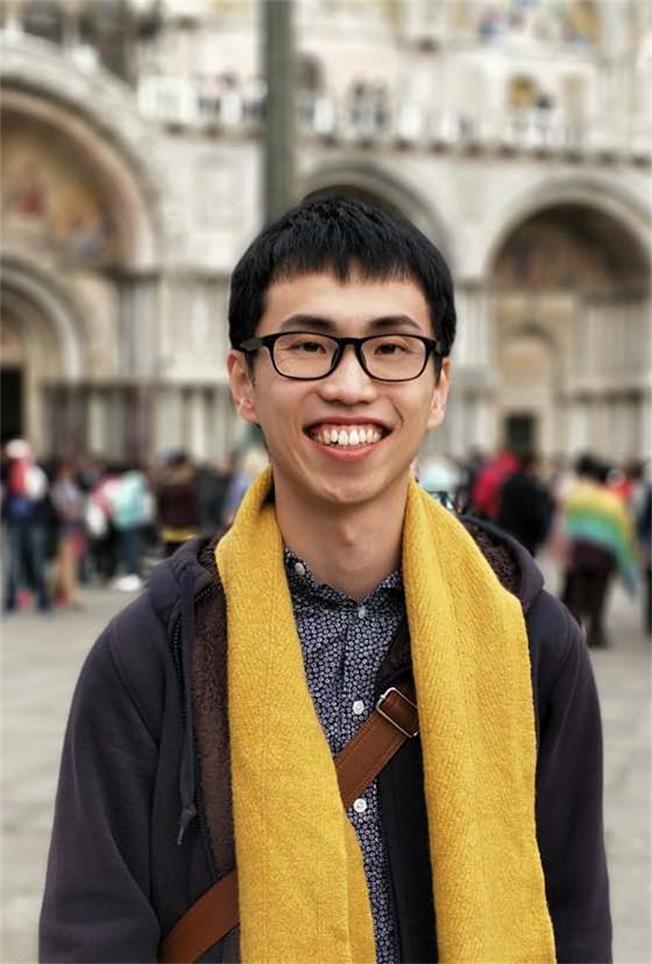}}]{Chunliang Li}
is a fifth-year Ph.D. student in Machine Learning Department at Carnegie Mellon University supervised by Prof. Barnabás Póczos. He is supported by IBM PhD fellowship. Chun-Liang's research is about deep generative models and representation learning. He has worked at Facebook and IBM as research interns. Prior to joining CMU, he received his B.S. and M.S. degree at National Taiwan University under supervision of Prof. Hsuan-Tien Lin.
\end{IEEEbiography}

\begin{IEEEbiography}
[{\includegraphics[width=1in, height=1.25in]{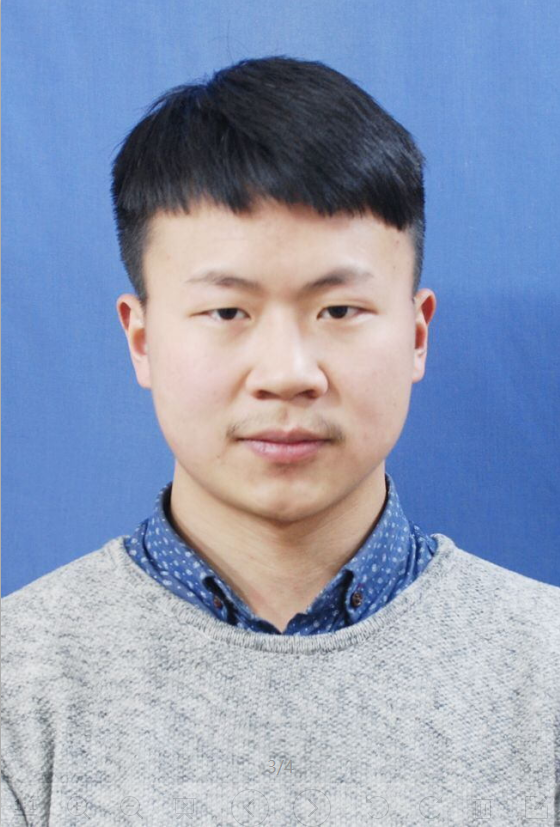}}]{Tianheng Cheng}
is currently a Ph.D. of Electronic Information and Communications, Huazhong University of Science and Technology, Wuhan, P.R. China. He was a research intern at Microsoft Research Asia and worked on object detection and image recognition. His research interests include computer vision, deep learning and reinforcement learning.
\end{IEEEbiography}

\begin{IEEEbiography}
[{\includegraphics[width=1in, height=1.25in]{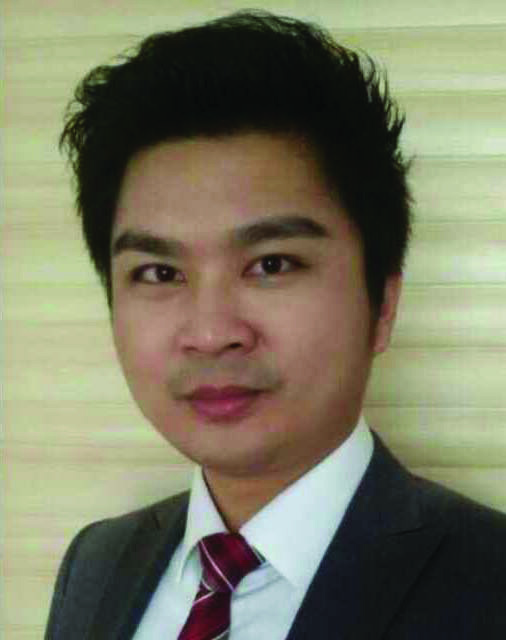}}]{Pan Zhou}
is currently an associate professor with School of Electronic Information and Communica-tions, Huazhong University of Science and Technology, Wuhan, P.R. China. He received his Ph.D. in the School of Electrical and Computer Engineering at the Georgia Institute of Technology (Georgia Tech) in 2011, Atlanta, USA. He received his B.S. degree in the Advanced Class of HUST, and a M.S. degree in the Department of Electronics and Information Engineering from HUST, Wuhan, China, in 2006 and 2008, respectively. He held honorary degree in his bachelor and merit research award of HUST in his master study. He was a senior technical member at Oracle Inc, America during 2011 to 2013, Boston, MA, USA, and worked on hadoop and distributed storage system for big data analytics at Oralce Cloud Platform. His current research interest includes: communication and information networks, security and privacy, machine learning and big data.
\end{IEEEbiography}

\newpage
\section{Appendix: Theoretical Analysis}
\label{theoretical_analysis}
We have introduced balance factor $\gamma$ into integrated cyclic framework. As mentioned above, the balance factor $\gamma$ accelerates the learning process of discriminator and thus controls model stability. Now we need to analyze how does this variation shift global optimum value of adversarial objective. For Bayesian cycleGAN based on whether standard GANs objective or Least Square GANs objective, we will first deduce the optimal status of discriminator. Furthermore, the global minimum of the training criterion and its equilibrium condition will be discussed in the round.

\subsection{Global Optimality of Standard-GANs-Objective Based Integrated Framework}
\label{sec:optim_ganloss}
\begin{proposition}
For any given generator $G_{A}$ and $G_{B}$,  the optimal discriminator $D_{A}$ ($D_{B}$ can be deduced analogously) comes in the following form:
\begin{align}
  D_{A}^{*}(y) &= \frac{(1 + \gamma)p_{y}(y)}{(1 + \gamma)p_{y}(y) + p_{\tilde{y}}(y) + \gamma p_{\hat{y}}(y)}.
\end{align}
\end{proposition}

\begin{proof}
  We use posterior item $p(\theta_{da} | t, \theta_{ga}, \theta_{gb})$, defined as Equ. (\ref{da}), to optimize $D_{A}$, aiming to maximum a posterior estimation. For convenience, we use the alternative posterior form which has the same meaning with original formula:
  \begin{align*}
  p(\theta_{da} | t, \theta_{ga}, \theta_{gb}) & \approx \sum_{k=1}^{m_y} \sum_{l=1}^{m_x} \left( \sum_{i=1}^{n_x} \sum_{j=1}^{n_y} p(\theta_{da} | \Omega_{da}, \theta_{ga}, \theta_{gb}) \right)
  \end{align*}
  \begin{align*}
   & \sum_{i=1}^{n_x} \sum_{j=1}^{n_y} p(\theta_{da} | \Omega_{da}, \theta_{ga}, \theta_{gb}) \\
   & \propto \prod_{j=1}^{n_y} (D_{A}(y^{(j)}; \theta_{da}))^{(1+\gamma)} \times \prod_{j=1}^{n_y} (1-D_{A}(\hat{y}_{z}^{(j)}; \theta_{da}))^{\gamma} \\
     & \times \prod_{i=1}^{n_x} (1-D_{A}(G_{A}(x_{z}^{(i)}; \theta_{ga}); \theta_{da})) \times p(\theta_{da} | \alpha_{da}),
  \end{align*}
  Due to the Monte Carlo sampling method \cite{hastings1970monte} has absolutely no effect on our proof \cite{smith2017don}, The optimization of $p(\theta_{da} | t, \theta_{ga}, \theta_{gb})$ equals to the optimization of $\sum_{i=1}^{n_x} \sum_{j=1}^{n_y} p(\theta_{da} | \Omega_{da}, \theta_{ga}, \theta_{gb})$, maximizing the quantity $V(G_{A},G_{B},D_{A})$:
  \begin{align*}
    \max_{D_{A}} & \, \log\prod_{i=1}^{n_{y}} ( D_{A}(y^{(i)}; \theta_{da}))^{1+\gamma} \times \prod_{i=1}^{n_{y}}(1-D_{A}(\hat{y}^{(i)}; \theta_{da}))^{\gamma} \\
    & \times \prod_{i=1}^{n_{x}}(1-D_{A}(G_{A}(x^{(i)}; \theta_{ga}); \theta_{da})),
  \end{align*}
 which can be expressed as:
  \begin{align*}
    & (1+\gamma) \cdot \mathbb{E}_{y \sim p_{y}(y)} [\log(D_{A}(y))] + \mathbb{E}_{y \sim p_{\tilde{y}}(y)} [\log(1-D_{A}(y))] 
    \\
    & + \gamma \cdot \mathbb{E}_{y \sim p_{\hat{y}}(y)} [\log(1-D_{A}(y))] \nonumber \\
    & = \int_{y} p_{y}(y)(1+\gamma)\log(D_{A}(y)) + p_{\tilde{y}}(y) \log(1-D_{A}(y)) 
    \\
    & + p_{\hat{y}}(y) \gamma \log(1-D_{A}(y)) \, dy. 
  \end{align*}
  The function $a(1+\gamma)\log(y)+b\log(1-y)+c\gamma \log(1-y)$ achieves its maximum in $[0,1]$ at $\frac{(1+\gamma)a}{(1+\gamma)a + b + \gamma c}$, so the integral function achieves the maximum at $D_{A}^{*}$.
\end{proof}

Given optimal status of discriminators, we want to maximize the posterior of generators, which is equivalent to minimize the posterior of two discriminators and minimize the L1 term $\left\| \hat{y} - y \right\|$, $\left\| \hat{x} - x \right\|$ simultaneously. The following theorem shows optimal minimum of objective and also demonstrates the introduced variation makes no substantial difference of equilibrium conditions.
\begin{thm}
The global minimum of the training criterion $C(G)+L1$ can be achieved if and only if $p_{y}(y) = p_{\tilde{y}}(y) = p_{\hat{y}}(y)$ and $p_{x}(x) = p_{\tilde{x}}(x) = p_{\hat{x}}(x)$.
\end{thm}

\begin{proof}
  \begin{align*}
    & C(G) = \max_{D_{A}} V(G_{A},G_{B},D_{A}) + \max_{D_{B}} V(G_{A},G_{B},D_{B}) \nonumber \\
    & = (1+\gamma) \cdot \mathbb{E}_{y \sim p_{y}(y)} [\log(D_{A}^{*}(y))] \\
    & + \mathbb{E}_{y \sim p_{\tilde{y}}(\tilde{y})} [\log(1-D_{A}^{*}(y))] + \gamma \cdot \mathbb{E}_{y \sim p_{\hat{y}}(\hat{y})} [\log(1-D_{A}^{*}(y))] \nonumber \\
    & + (1+\gamma) \cdot \mathbb{E}_{x \sim p_{x}(x)} [\log(D_{B}^{*}(x))] \\
    & + \mathbb{E}_{x \sim p_{\tilde{x}}(\tilde{x})} [\log(1-D_{B}^{*}(x))] + \gamma \cdot \mathbb{E}_{x \sim p_{\hat{x}}(\hat{x})} [\log(1-D_{B}^{*}(x))] \nonumber \\
    & = (1+\gamma)\sum p_{y} \log(\frac{(1+\gamma)p_{y}}{(1+\gamma)p_{y}+p_{\tilde{y}}+\gamma p_{\hat{y}}}) \\ 
    & + \sum (p_{\tilde{y}} + \gamma p_{\hat{y}}) \log(\frac{p_{\tilde{y}}+\gamma p_{\hat{y}}}{(1+\gamma)p_{y}+p_{\tilde{y}}+\gamma p_{\hat{y}}}) \nonumber \\
    & \, + (1+\gamma)\sum p_{x} \log(\frac{(1+\gamma)p_{x}}{(1+\gamma)p_{x}+p_{\tilde{x}}+\gamma p_{\hat{x}}}) \\
    & + \sum (p_{\tilde{x}} + \gamma p_{\hat{x}}) \log(\frac{p_{\tilde{x}}+\gamma p_{\hat{x}}}{(1+\gamma)p_{x}+p_{\tilde{x}}+\gamma p_{\hat{x}}}). \nonumber
  \end{align*}
  When the given condition is tenable, L1 term is zero.
  We rewrite the $C(G)$ equation with Kullback-Leibler divergence form and deduce that the minimum value of $C(G)$ is $-4 (1 + \gamma) \log 2$.
  \begin{align*}
    C(G) & = - 2(1+\gamma) \log 2 - 2(1+\gamma) \log 2 \nonumber \\
    & + KL\left((1+\gamma)p_{y} \left \| \frac{(1+\gamma)p_{y} + p_{\tilde{y}} + \gamma p_{\hat{y}}}{2} \right. \right) \\
    & + KL\left(p_{\tilde{y}} + \gamma p_{\hat{y}} \left \| \frac{(1+\gamma)p_{y} + p_{\tilde{y}} + \gamma p_{\hat{y}}}{2} \right. \right) \nonumber \\
    & + KL\left((1+\gamma)p_{x} \left \| \frac{(1+\gamma)p_{x} + p_{\tilde{x}} + \gamma p_{\hat{x}}}{2} \right. \right) \\
    & + KL\left(p_{\tilde{x}} + \gamma p_{\hat{x}} \left \| \frac{(1+\gamma)p_{x} + p_{\tilde{x}} + \gamma p_{\hat{x}}}{2} \right. \right) \nonumber \\
    & = -4 (1+\gamma) \log 2 + 2 \cdot JSD \left((1+\gamma)p_{y} \left \| (p_{\tilde{y}} + \gamma p_{\hat{y}}) \right.\right)\\
    & + 2 \cdot JSD \left((1+\gamma)p_{x} \left \| (p_{\tilde{x}} + \gamma p_{\hat{x}}) \right.\right). \nonumber
  \end{align*}
  According to the non-negative property of Jensen-Shannon divergence, the $JSD$ can be zero only when that two distributions are equal, that is $(1+\gamma)p_{y}(y)=p_{\tilde{y}}(y) + \gamma p_{\hat{y}}(y)$ and $(1+\gamma)p_{x}(x)=p_{\tilde{x}}(x) + \gamma p_{\hat{x}}(x)$. Then take the L1 term into consideration, if we want that achieve its minimum, there must be $p_{\hat{y}} = p_{y}$ and $p_{\hat{x}} = p_{x}$, so the global minimum of training criterion $C(G)+L1$ can be achieved if and only if $p_{y}(y) = p_{\tilde{y}}(y) = p_{\hat{y}}(y)$ and $p_{x}(x) = p_{\tilde{x}}(x) = p_{\hat{x}}(x)$.
\end{proof}

\subsection{Global Optimality of Least-Square-GANs-Objective Based Integrated Framework}
\label{sec:optim_lsganloss}
\begin{proposition}
For any given generator $G_{A}$ and $G_{B}$, the optimal discriminator $D_{A}$ ($D_{B}$ can be deduced analogously) comes in the following form:
\begin{align}
  D_{A}^{*}(y) &= \frac{(1 + \gamma)p_{y}(y)}{(1 + \gamma)p_{y}(y) + p_{\tilde{y}}(y) + p_{\hat{y}}(y)}.
\end{align}
\end{proposition}

\begin{proof}
  We use posterior item $p(\theta_{da} | t, \theta_{ga}, \theta_{gb})$, defined as Equ. (\ref{da-ls}), to optimize $D_{A}$, aiming to maximum a posterior estimation, due to the same reason with standard GAN objective based deduction, the optimization of $p(\theta_{da} | t, \theta_{ga}, \theta_{gb})$ equals to the optimization of \\
  $\sum_{i=1}^{n_x} \sum_{j=1}^{n_y} p(\theta_{da} | \Omega_{da}, \theta_{ga}, \theta_{gb})$, minimizing the quantity $V(G_{A},G_{B},D_{A})$:
  \begin{align*}
    \max_{D_{A}} \,\, & \log\prod_{i=1}^{n_{y}} e^{-(1+\gamma) \cdot (D_{A}(y^{(i)}; \theta_{da}) - 1)^{2} -\gamma \cdot (D_{A}( \hat{y}^{(i)}; \theta{da}))^{2}} \\
    & \times \prod_{i=1}^{n_{x}} e^{-(D_{A}( G_{A}(x^{(i)}; \theta_{ga}); \theta_{da}))^{2}},
  \end{align*}
  which can be expressed as:
  \begin{align*}
    & V(G_{A},G_{B},D_{A}) \\
    & = (1+\gamma) \cdot \mathbb{E}_{y \sim p_{y}(y)} [(D_{A}(y)-1)^{2}] + \mathbb{E}_{y \sim p_{\tilde{y}}(y)} [(D_{A}(y))^{2}] \\
    & + \gamma \cdot \mathbb{E}_{y \sim p_{\hat{y}}(y)} [(D_{A}(y))^{2}] \\
    & = \int_{y} p_{y}(y)(1+\gamma)(D_{A}(y)-1)^{2} + p_{\tilde{y}}(y) (D_{A}(y))^{2} \\
    & + p_{\hat{y}}(y) \gamma (D_{A}(y))^{2} \, dy.
  \end{align*}
  The function $a(1+\gamma)(y - 1)^{2} + by^{2}+c\gamma y^{2}$ achieves its minimum in $[0,1]$ at $\frac{(1+\gamma)a}{(1+\gamma)a + b + \gamma c}$, so the integral function achieves the minimum value at $D_{A}^{*}$.
\end{proof}

Given optimal status of discriminators, we want to maximize the posteriors of generators, which is related to the distribution of generated $\tilde{y}, \hat{y}$ and $\tilde{x}, \hat{x}$. The maximize process can be converted to minimize $C(G)$ and L1 term $\left\| \hat{y} - y \right\|,$ $\left\| \hat{x} - x \right\|$ simultaneously. The following lemma shows optimal minimum of objective and also demonstrates the balance factor has no substantial effect on equilibrium conditions.

\begin{thm}
The global minimum of the training criterion C(G) can be achieved if and only if $p_{y}(y) = p_{\tilde{y}}(y) = p_{\hat{y}}(y)$ and $p_{x}(x) = p_{\tilde{x}}(x) = p_{\hat{x}}(x)$.
\end{thm}

\begin{proof}
  \begin{align*}
    & C(G) = (1+\gamma) \cdot \mathbb{E}_{y \sim p_{y}(y)} [(D_{A}^{*}(y)-1)^{2}] \\
    & + \mathbb{E}_{y \sim p_{\tilde{y}}(y)} [(D_{A}^{*}(y))^{2}] + \gamma \cdot \mathbb{E}_{y \sim p_{\hat{y}}(y)} [(D_{A}^{*}(y))^{2}] \nonumber \\
    & + (1+\gamma) \cdot \mathbb{E}_{x \sim p_{x}(x)} [(D_{A}^{*}(x)-1)^{2}] \\
    & + \mathbb{E}_{x \sim p_{\tilde{x}}(x)} [(D_{A}^{*}(x))^{2}] + \gamma \cdot \mathbb{E}_{x \sim p_{\hat{x}}(x)} [(D_{A}^{*}(x))^{2}] \\
    & = \int_{y} (1+\gamma)p_{y}(\frac{(1+\gamma)p_{y}}{(1+\gamma)p_{y} + p_{\tilde{y}} + \gamma p_{\hat{y}}} - 1)^{2} \\
    & + (p_{\tilde{y}} + \gamma p_{\hat{y}})(\frac{(1+\gamma)p_{y}}{(1+\gamma)p_{y} + p_{\tilde{y}} + \gamma p_{\hat{y}}})^{2} \, dy \\
    & + \int_{x} (1+\gamma)p_{x}(\frac{(1+\gamma)p_{x}}{(1+\gamma)p_{x} + p_{\tilde{x}} + \gamma p_{\hat{x}}} - 1)^{2} \\
    & + (p_{\tilde{x}} + \gamma p_{\hat{y}})(\frac{(1+\gamma)p_{x}}{(1+\gamma)p_{x} + p_{\tilde{x}} + \gamma p_{\hat{x}}})^{2} \, dx \\
    & = \! \int_{y} \! \frac{(1+\gamma)p_{y}(p_{\tilde{y}} + \gamma p_{\hat{y}})}{(1+\gamma)p_{y} + p_{\tilde{y}} + \gamma p_{\hat{y}}} dy \! + \!\! \int_{x} \! \frac{(1+\gamma)p_{x}(p_{\tilde{x}} + \gamma p_{\hat{x}})}{(1+\gamma)p_{x} + p_{\tilde{x}} + \gamma p_{\hat{x}}} \, dx. \nonumber
  \end{align*}
  When the given condition is tenable, L1 term is zero, we can get the minimum value of $C(G)$ as $(1 + \gamma)$.
  Since $\frac{1}{x + 1} - \frac{1}{2}$ is a convex function when $x > 0$ and reach zero when $x = 1$, therefore the above equation is a valid $f$-divergence.
  \begin{align*}
  	&C(G) = 1 + \gamma + \int_{y} ( \frac{1}{1+\frac{(1+\gamma)p_{y}}{p_{\tilde{y}} + p_{\hat{y}}}} - \frac{1}{2}) \times (1+\gamma)p_{y} \, dy  \\
  	& + \int_{x} ( \frac{1}{1+\frac{(1+\gamma)p_{x}}{p_{\tilde{x}} + p_{\hat{x}}}} - \frac{1}{2}) \times (1+\gamma)p_{x} \, dx \nonumber \\
    & = 1 + \gamma + \mathcal{D}_{f} \left( (1+\gamma)p_{y} \left\| (p_{\tilde{y}} + p_{\hat{y}}) \right. \right) \\
    & + \mathcal{D}_{f} \left( (1+\gamma)p_{x} \left\| (p_{\tilde{x}} + p_{\hat{x}}) \right. \right).
  \end{align*}
According to the non-negative property of $f$-divergence, the $\mathcal{D}_{f}$ can be zero only when that two distribution equal, that is $(1+\gamma)p_{y}(y)=p_{\tilde{y}}(y) + \gamma p_{\hat{y}}(y)$ and $(1+\gamma)p_{x}(x)=p_{\tilde{x}}(x) + \gamma p_{\hat{x}}(x)$. Then take the L1 term into consideration, if we want the criterion achieve its minimum, there must be $p_{\hat{y}} = p_{y}$ and $p_{\hat{x}} = p_{x}$, so the global minimum of training criterion $C(G)+L_{1}$ can be achieved if and only if $p_{y}(y) = p_{\tilde{y}}(y) = p_{\hat{y}}(y)$ and $p_{x}(x) = p_{\tilde{x}}(x) = p_{\hat{x}}(x)$.
\end{proof}

\end{document}